\documentclass[11pt, a4paper]{article}

\usepackage[margin=1.19in]{geometry}
\usepackage[utf8]{inputenc} 
\usepackage[T1]{fontenc}    

\usepackage[colorlinks=true,linkcolor=blue,citecolor=blue]{hyperref}       
\usepackage{url}            
\usepackage{booktabs}      
\usepackage{nicefrac}     
\usepackage{algorithm}
\usepackage[noend]{algorithmic} 
\usepackage{times}
\usepackage{latexsym}
\usepackage{graphicx}
\usepackage{wrapfig}
\usepackage{comment} 
\usepackage{subcaption}
\usepackage[english]{babel} 
\usepackage{amsmath,amsfonts,amsthm} 
\usepackage{enumitem}
\usepackage{xcolor}
\usepackage{amssymb}
\usepackage{thmtools}
\usepackage{mathtools}
\usepackage{nameref,hyperref,cleveref}
\usepackage{natbib}
\usepackage{multirow}
\usepackage{authblk}

\newtheorem{theorem}{Theorem}
\declaretheorem[name=Proposition, sibling=theorem, refname={proposition,propositions}, Refname={Proposition,Propositions}]{proposition}
\newtheorem{lemma}{Lemma}
\newtheorem{remark}{Remark}

\newtheorem{definition}{Definition}
\newtheorem{corollary}{Corollary}

\newcommand{\dacop}{T^{dac}_{\widetilde \Acal, \Xb}}
\newcommand{\dau}{d_{\textit{aug}}}

\input{definitions}
\input{format}

\usepackage{hyperref}
\usepackage{url}

\title{Sample Efficiency of Data Augmentation Consistency Regularization}

\author[1]{Shuo Yang\thanks{Equal contribution. Correspondence to: \href{mailto:yangshuo_ut@utexas.edu}{yangshuo\_ut@utexas.edu}, \href{mailto:ydong@utexas.edu}{ydong@utexas.edu}, \href{mailto:qilei@princeton.edu}{qilei@princeton.edu}}}
\author[1]{Yijun Dong$^*$}
\author[1]{Rachel Ward}
\author[1]{Inderjit S. Dhillon}
\author[1]{Sujay Sanghavi}
\author[2]{\\Qi Lei}
\affil[1]{The University of Texas at Austin}
\affil[2]{Princeton University}

\begin{document}

\maketitle
\begin{abstract}

Data augmentation is popular in the training of large neural networks; currently, however, there is no clear theoretical comparison between different algorithmic choices on how to use augmented data. In this paper, we take a step in this direction -- we first present a simple and novel analysis for linear regression with label invariant augmentations, demonstrating that data augmentation consistency (DAC) is intrinsically more efficient than empirical risk minimization on augmented data (DA-ERM). The analysis is then extended to misspecified augmentations (i.e., augmentations that change the labels), which again demonstrates the merit of DAC over DA-ERM. 
Further, we extend our analysis to non-linear models (e.g., neural networks) and present generalization bounds. Finally, we perform experiments that make a clean and apples-to-apples comparison (i.e., with no extra modeling or data tweaks) between DAC and DA-ERM using CIFAR-100 and WideResNet; these together demonstrate the superior efficacy of DAC.

\end{abstract}
\section{Introduction}

Modern machine learning models, especially deep learning models, require abundant training samples. Since data collection and human annotation are expensive, data augmentation has become a ubiquitous practice in creating artificial labeled samples and improving generalization performance. This practice is corroborated by the fact that the semantics of images remain the same through simple translations like obscuring, flipping, rotation, color jitter, rescaling~\citep{shorten2019survey}. Conventional algorithms use data augmentation to expand the training data set~\citep{simard1998transformation,krizhevsky2012imagenet,simonyan2014very,he2016deep,cubuk2018autoaugment}. 

Data Augmentation Consistency (DAC) regularization, as an alternative, enforces the model to output similar predictions on the original and augmented samples and has contributed to many recent state-of-the-art supervised or semi-supervised algorithms. This idea was first proposed in~\cite{bachman2014learning} and popularized by~\cite{laine2016temporal,sajjadi2016regularization}, and gained more attention recently with the success of FixMatch~\citep{sohn2020fixmatch} for semi-supervised few-shot learning as well as AdaMatch~\citep{berthelot2021adamatch} for domain adaptation. DAC can utilize unlabeled samples, as one can augment the training samples and enforce consistent predictions without knowing the true labels. This bypasses the limitation of the conventional algorithms that can only augment labeled samples and add them to the training set (referred to as DA-ERM). However, it is not well-understood whether DAC has additional algorithmic benefits compared to DA-ERM. We are, therefore, seeking a theoretical answer. 

Despite the empirical success, the theoretical understanding of data augmentation (DA) remains limited. Existing work \citep{chen2020group,mei2021learning,lyle2019analysis} focused on establishing that augmenting data saves on the number of labeled samples needed for the same level of accuracy. However, none of these explicitly compare the efficacy (in terms of the number of augmented samples) between different algorithmic choices on {\em how to use the augmented samples} in an apples-to-apples way.

In this paper, we focus on the following research question:
\begin{center}

\textit{Is DAC intrinsically more efficient than DA-ERM (even without unlabeled samples)? }
\end{center}

We answer the question affirmatively. We show that DAC is intrinsically more efficient than DA-ERM with a simple and novel analysis for linear regression under label invariant augmentations. We then extend the analysis to misspecified augmentations (i.e., those that change the labels). We further provide generalization bounds under consistency regularization for non-linear models like two-layer neural networks and DNN-based classifiers with expansion-based augmentations. Intuitively, we show DAC is better than DA-ERM in the following sense: 1) DAC enforces stronger invariance in the learned models, yielding smaller estimation error; and 2) DAC better tolerates mis-specified augmentations and incurs smaller approximation error. Our theoretical findings can also explain and guide some technical choices, e.g. why we can use stronger augmentation in consistency regularization but only weaker augmentation when creating pseudo-labels~\citep{sohn2020fixmatch}.  

Specifically, our \textbf{main contributions} are:
\begin{itemize}
    \item \textbf{Theoretical comparisons between DAC and DA-ERM.} We first present a simple and novel result for linear regression, which shows that DAC yields a strictly smaller generalization error than DA-ERM using the same augmented data. Further, we demonstrate that with with the flexibility of hyper-parameter tuning, DAC can better handle data augmentation with small misspecification in the labels. 
    \item \textbf{Extended analysis for non-linear models.} 
    We derive generalization bounds for DAC under two-layer neural networks, and classification with expansion-based augmentations. 
    \item \textbf{Empirical comparisons between DAC and DA-ERM.} We perform experiments that make a clean and apples-to-apples comparison (i.e., with no extra modeling or data tweaks) between DAC and DA-ERM using CIFAR-100 and WideResNet. Our empirical results demonstrate the superior efficacy of DAC.
\end{itemize}

\section{Related Work}

\textbf{Empirical findings. }
Data augmentation (DA) is an essential ingredient for almost every state-of-the-art supervised learning algorithm since the seminal work of \cite{krizhevsky2012imagenet} (see reference therein \citep{simard1998transformation,simonyan2014very,he2016deep,cubuk2018autoaugment,kuchnik2018efficient}). It started from adding augmented data to the training samples via (random) perturbations, distortions, scales, crops, rotations, and horizontal flips. More sophisticated variants were subsequently designed; a non-exhaustive list includes Mixup \citep{zhang2017mixup}, Cutout \citep{devries2017improved}, and Cutmix \citep{yun2019cutmix}. The choice of data augmentation and their combinations require domain knowledge and experts' heuristics, which triggered some automated search algorithms to find the best augmentation strategies~\citep{lim2019fast,cubuk2019autoaugment}. The effects of different DAs are systematically explored in \cite{tensmeyer2016improving}. 

Recent practices not only add augmented data to the training set but also enforce similar predictions by adding consistency regularization~\citep{bachman2014learning,laine2016temporal,sohn2020fixmatch}. One benefit of DAC is the feasibility of exploiting unlabeled data. Therefore input consistency on augmented data also formed a major component to state-of-the-art algorithms for semi-supervised learning~\citep{laine2016temporal,sajjadi2016regularization,sohn2020fixmatch,xie2020self}, self-supervised learning~\citep{chen2020simple}, and unsupervised domain adaptation~\citep{french2017self,berthelot2021adamatch}.

\textbf{Theoretical studies. }
Many interpret the effect of DA as some form of regularization~\citep{he2019data}. Some work focuses on linear transformations and linear models \citep{wu2020generalization} or kernel classifiers \citep{dao2019kernel}. Convolutional neural networks by design enforce translation equivariance symmetry \citep{benton2020learning,li2019enhanced};  further studies have hard-coded CNN's invariance or equivariance to rotation~\citep{cohen2016group,marcos2017rotation,worrall2017harmonic,zhou2017oriented}, scaling~\citep{sosnovik2019scale,worrall2019deep} and other types of transformations.

Another line of work views data augmentation as invariant learning by averaging over group actions~\citep{lyle2019analysis,chen2020group,mei2021learning}. They consider an ideal setting that is equivalent to ERM with all possible augmented data, bringing a clean mathematical interpretation. We are interested in a more realistic setting with limited augmented data. In this setting, it is crucial to utilize the limited data with proper training methods, the difference of which cannot be revealed under previously studied settings.  

Some more recent work investigates the feature representation learning procedure with DA for self-supervised learning tasks~\citep{garg2020functional,wen2021toward,haochen2021provable,von2021self}. \cite{cai2021theory,wei2021theoretical} studied the effect of data augmentation with label propagation. Data augmentation is also deployed to improve robustness \citep{rajput2019does}, to facilitate domain adaptation and domain generalization~\citep{cai2021theory,sagawa2019distributionally}.

\section{Problem Setup and Data Augmentation Consistency}\label{sec:general}

Consider the standard supervised learning problem setup: $\xb \in \Xcal$ is input feature, and $y \in \Ycal$ is its label (or response). Let $\Pgt$ be the true distribution of $\rbr{\xb, y}$ (i.e., the label distribution follows $y \sim \Pgt(y | \xb)$). We have the following definition for label invariant augmentation.

\begin{definition}[Label Invariant Augmentation] \label{def:causal_invar_data_aug}
For any sample $\xb \in \Xcal$, we say $A(\xb)\in \Xcal$ is a label invariant augmentation if and only if $\Pgt(y | \xb) = \Pgt(y | A(\xb))$.
\end{definition}

Our work largely relies on label invariant augmentation but also extends to augmentations that incur small misspecification in their labels. Therefore our results apply to the augmentations achieved via certain transformations (e.g., random cropping, rotation), and we do not intend to cover augmentations that can largely alter the semantic meanings (e.g., MixUp \citep{zhang2017mixup}).

Now we introduce the learning problem on an augmented dataset. Let $(\Xb, \yb) \in \Xcal^N \times \Ycal^N$ be a training set consisting of $N$ $\iid$ samples. Besides the original $\rbr{\Xb,\yb}$, each training sample is provided with $\alpha$ augmented samples. The features of the augmented dataset $\widetilde\Acal(\xb) \in \Xcal^{(1+\alpha) N}$ is:
\begin{align*}
    \widetilde\Acal(\Xb) = \sbr{\xb_{1}; \cdots; \xb_{N}; \xb_{1,1}; \cdots; \xb_{N,1}; \cdots; \xb_{1, \alpha}; \cdots; \xb_{N, \alpha}} \in \Xcal^{(1+\alpha) N},
\end{align*}
where $\xb_i$ is in the original training set and $\xb_{i, j}, \forall j \in [\alpha]$ are the augmentations of $\xb_i$. The labels of the augmented samples are kept the same, which can be denoted as $\widetilde \Mb \yb \in \Ycal^{(1+\alpha)N}$, where $\wt\Mb \in \RR^{(1+\alpha) N \times N}$ is a vertical stack of $(1+\alpha)$ identity mappings. 

\textbf{Data Augmentation Consistency Regularization.} Let $\Hcal = \cbr{h: \Xcal \rightarrow \Ycal}$ be a well-specified function class (e.g., for linear regression problems, $\exists h^* \in \Hcal$, s.t. $h^*(\xb) = \EE[y |\xb]$) that we hope to learn from. Without loss of generality, we assume that each function $h \in \Hcal$ can be expressed as $h = f_h \circ \phi_h$, where $\phi_h \in \Phi = \cbr{\phi: \Xcal \rightarrow \Wcal}$ is a proper representation mapping and $f_h \in \Fcal = \cbr{f:\Wcal \rightarrow \Ycal}$ is a predictor on top of the learned representation. We tend to decompose $h$ such that $\phi_h$ is a powerful feature extraction function whereas $f_h$ can be as simple as a linear combiner. For instance, in a deep neural network, all the layers before the final layer can be viewed as feature extraction $\phi_h$, and the predictor $f_h$ is the final linear combination layer.

For a loss function $l: \Ycal \times \Ycal \rightarrow \RR$ and a metric $\varrho$ properly defined on the representation space $\Wcal$, learning with data augmentation consistency (DAC) regularization is:
\begin{align}\label{eq:dac_soft}
    \argmin_{h\in\Hcal}\sum_{i=1}^{N}l(h(\xb_i), y_i) + \underbrace{\lambda\sum_{i=1}^N\sum_{j=1}^{\alpha} \varrho\rbr{\phi_h(\xb_i), \phi_h(\xb_{i, j})}}_{\textit{DAC regularization}}.
\end{align}

Note that the DAC regularization in \Cref{eq:dac_soft} can be easily implemented empirically as a regularizer. Intuitively, DAC regularization penalizes the representation difference between the original sample $\phi_h(\xb_i)$ and the augmented sample $\phi_h(\xb_{i,j})$, with the belief that similar samples (i.e., original and augmented samples) should have similar representations. When the data augmentations do not alter the labels, it is reasonable to enforce a strong regularization (i.e., $\lambda \rightarrow \infty$) -- since the conditional distribution of $y$ does not change. The learned function $\widehat h^{dac}$ can then be written as the solution of a constrained optimization problem:
\begin{align}\label{eq:dac_hard}
\begin{split}
    & \widehat h^{dac} \triangleq \argmin_{h \in \Hcal} \sum_{i=1}^N l(h(\xb_i), y_i)\quad \text{s.t.}\quad \phi_h(\xb_i) = \phi_h(\xb_{i,j}),~ \forall i \in [N], j \in [\alpha].
\end{split}
\end{align}

In the rest of the paper, we mainly focus on the data augmentations satisfying \Cref{def:causal_invar_data_aug} and our analysis relies on the formulation of \Cref{eq:dac_hard}. When the data augmentations alter the label distributions (i.e., not satisfying \Cref{def:causal_invar_data_aug}), it becomes necessary to adopt a finite $\lambda$ for \Cref{eq:dac_soft}, and such extension is discussed in \Cref{subsec:finite_lambda}.
\section{Linear Model and Label Invariant Augmentations}\label{sec:linear_regression_label_invariant}

In this section, we show the efficacy of DAC regularization with linear regression under label invariant augmentations (\Cref{def:causal_invar_data_aug}). 

To see the efficacy of DAC regularization (i.e., \Cref{eq:dac_hard}), we revisit a more commonly adopted training method here -- empirical risk minimization on augmented data (DA-ERM):
\begin{align}\label{eq:plain_erm}
    \wh h^{\herm} \triangleq \argmin_{h\in \Hcal} \sum_{i=1}^N l (h(\xb_i), y_i) + \sum_{i=1}^N\sum_{j=1}^{\alpha}l (h(\xb_{i,j}), y_i).
\end{align}
Now we show that the DAC regularization (\Cref{eq:dac_hard}) learns more efficiently than DA-ERM. Consider the following setting: given $N$ observations $\Xb \in \RR^{N \times d}$, the responses $\yb \in \RR^N$ are generated from a linear model $\yb = \Xb\thetab^* + \epsb$, where $\epsb\in \RR^N$ is zero-mean noise with $\EE\sbr{\epsb\epsb^\top} = \sigma^2 \Ib_N$. Recall that $\widetilde \Acal(\Xb)$ is the entire augmented dataset, and $\widetilde \Mb \yb$ corresponds to the labels. We focus on the fixed design excess risk of $\thetab$ on $\widetilde \Acal(\Xb)$, which is defined as $L(\thetab) \triangleq \frac{1}{(1+\alpha) N}\norm{\widetilde \Acal(\Xb)\thetab - \widetilde \Acal(\Xb)\thetab^*}_2^2$.

Let $\dau \triangleq \rank\rbr{\widetilde \Acal(\Xb) - \widetilde \Mb \Xb}$ measure the number of dimensions perturbed by augmentation (i.e., large $\dau$ implies that $\wt\Acal(\Xb)$ well perturbs the original dataset). Assuming that $\wt\Acal(\Xb)$ has full column rank (such that the linear regression problem has a unique solution), we have the following result for learning by DAC versus DA-ERM.

\begin{theorem}[Informal result on linear regression (formally in \Cref{thm:formal_linear_regression})]\label{thm:informal_linear_regression}
    Learning with DAC regularization, we have $\EE_{\epsb}\sbr{L(\widehat \thetab^{dac}) - L(\thetab^*)} = \frac{(d - \dau)\sigma^2}{N}$, while learning with ERM directly on the augmented dataset, we have $\EE_{\epsb}\sbr{L(\widehat \thetab^{\herm}) - L(\thetab^*)} = \frac{(d - \dau + d')\sigma^2}{N}$, where $d' \in [0, \dau]$. 
\end{theorem}

Formally, $d'$ is defined as $d' \triangleq \frac{\tr\rbr{ \rbr{\projAX - \Pb_{\Scal}} {\wt\Mb \widetilde \Mb^\top} }}{1+\alpha}$, where $\projAX \triangleq \widetilde \Acal(\Xb) \widetilde \Acal(\Xb)^\pinv$, and $\Pb_\Scal$ is the projector onto $\Scal \triangleq \cbr{\widetilde \Mb \Xb \thetab~|~\forall \thetab \in \RR^d, s.t. \rbr{\widetilde\Acal(\Xb) - \widetilde \Mb \Xb}\thetab = 0}$. Under standard conditions (e.g., $\xb$ is sub-Gaussian and $N$ is not too small), it is not hard to extend \Cref{thm:informal_linear_regression} to random design (i.e., the more commonly acknowledged generalization bound) with the same order.

\begin{remark}[Why DAC is more effective]
    Intuitively, DAC reduces the dimensions from $d$ to $d - \dau$ by enforcing consistency regularization. DA-ERM, on the other hand, still learns in the original $d$-dimensional space. $d'$ characterizes such difference.
\end{remark}

\begin{wrapfigure}{t}{0.5\columnwidth}
\vspace{-1em} %1.2
	\centering
	\includegraphics[width=\linewidth]{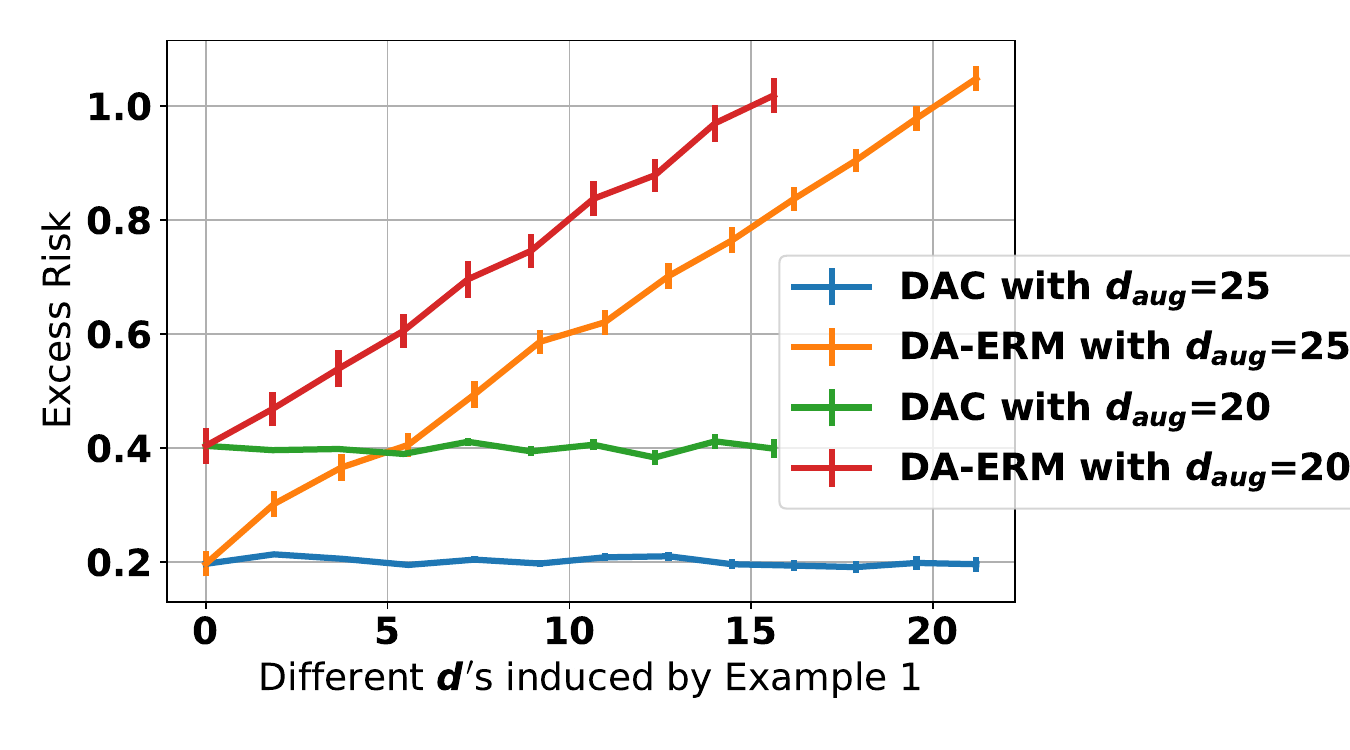}
	\vspace{-2em}
	\caption{Comparison of DAC regularization and DA-ERM (\Cref{example:linear_regression}). The results precisely match \Cref{thm:informal_linear_regression}. DA-ERM depends on the $d'$ induced by different augmentations, while the DAC regularization works equally well for all $d'$ and better than the DA-ERM. Further, both DAC and DA-ERM are affected by $\dau$, the number of dimensions perturbed by $\wt\Acal(\Xb)$.}
\label{fig:linear_regression}
\vspace{-3em}
\end{wrapfigure}

Here we present an explanation for $d'$. Note that $\sigma^2 \cdot {\wt \Mb \widetilde \Mb^\top}$ is the noise covariance matrix of the augmented dataset. $\tr\rbr{\Pb_\Scal {\wt\Mb \widetilde \Mb^\top}}$ is the variance of $\widehat\thetab^{dac}$, and $\tr\rbr{\projAX {\wt\Mb \widetilde \Mb^\top}}$ is the variance of $\widehat\thetab^{\herm}$. Thus, $d' \propto \tr\rbr{\rbr{\projAX - \Pb_\Scal} {\wt\Mb \widetilde \Mb^\top}}$ measures the difference. When $\projAX \neq \Pb_\Scal$ (a common scenario as instantiated in \Cref{example:linear_regression}), DAC is strictly better than DA-ERM.

\begin{example}\label{example:linear_regression} Consider a 30-dimensional linear regression. The original training set contains 50 samples. The inputs $\xb_i$s are generated independently from $\Ncal(0, \Ib_{30})$ and we set $\thetab^* = [\thetab_c^*; \mathbf{0}]$ with $\thetab_c^* \sim \Ncal(0, \Ib_5)$  and $\mathbf{0} \in \RR^{25}$. The noise variance $\sigma$ is set to $1$. We partition $\xb$ into 3 parts $[x_{c1}, x_{e_1}, x_{e2}]$ and take the following augmentations: $A([x_{c1}; x_{e1}; x_{e2}]) = [x_{c1}; 2x_{e1}; -x_{e2}], x_{c1}\in\RR^{d_{c1}}, x_{e1}\in \RR^{d_{e1}}, x_{e2}\in\RR^{d_{e2}}$, where $d_{c1} + d_{e1} + d_{e2} = 30$. 

Notice that the augmentation perturbs $x_{e1}$ and $x_{e2}$ and leaving $x_{c1}$ unchanged, we therefore have $\dau = 30 - d_{c1}$. By changing $d_{c1}$ and $d_{e1}$, we can have different augmentations with different $\dau, d'$. The results for $\dau \in \cbr{20, 25}$ and various $d'$s are presented in \Cref{fig:linear_regression}. The excess risks precisely match \Cref{thm:informal_linear_regression}. It confirms that the DAC regularization is strictly better than DA-ERM for a wide variety of augmentations.
\end{example}

\section{Beyond Label Invariant Augmentation}\label{subsec:finite_lambda}

In this section, we extend our analysis to misspecified augmentations by relaxing the label invariance assumption (such that $\Pgt(y | \xb) \neq \Pgt(y | A(\xb))$). With an illustrative linear regression problem, we show that DAC also brings advantages over DA-ERM for misspecified augmentations. 

We first recall the linear regression setup: given a set of $N$ $\iid$ samples $\rbr{\Xb, \yb}$ that follows $\yb = \Xb \thetab^* + \epsb$ where $\epsb$ are zero-mean independent noise with $\E\sbr{\epsb \epsb^\top} = \sigma^2 \Ib_N$, we aim to learn the unknown ground truth $\thetab^*$. For randomly generated misspecified augmentations $\wt \Acal(\Xb)$ that alter the labels (\ie, $\wt\Acal\rbr{\Xb}\thetab^* \neq \wt\Mb\Xb\thetab^*$), a proper consistency constraint is $\norm{\phi_h(\xb_i) - \phi_h(\xb_{i, j})}_2 \le \constmis$ (where $\xb_{i, j}$ is an augmentation of $\xb_i$, noticing that $\constmis = 0$ corresponds to label invariant augmentations in \Cref{def:causal_invar_data_aug}). 
For $\constmis>0$, the constrained optimization is equivalent to:
\begin{align}\label{eq:dac_soft_reg}
    \wh\thetab^{dac} = \argmin_{\thetab \in \R^d} \frac{1}{N}\nbr{\Xb \thetab - \yb}_2^2 
    + \frac{\lambda }{\rbr{1+\alpha} N} \norm{\rbr{\wt\Acal\rbr{\Xb} - \wt\Mb\Xb} \thetab}_2^2
,\end{align}
for some finite $0 < \lambda < \infty$. We compare $\wh\thetab^{dac}$ to the solution learned with ERM on augmented data (as in \Cref{eq:plain_erm}):
\begin{align*}
    \wh\thetab^{\herm} = \argmin_{\thetab \in \R^d} \frac{1}{\rbr{1+\alpha}N} \nbr{\wt\Acal\rbr{\Xb} \thetab - \wt\Mb \yb}_2^2.
\end{align*} 
Let $\covtr \triangleq \frac{1}{N} \Xb^\top \Xb$ and $\covall \triangleq \frac{\wt\Acal\rbr{\Xb}^\top \wt\Acal\rbr{\Xb}}{(1+\alpha)N}$. 
With $\Sb = \frac{\wt\Mb^\top \wt\Acal\rbr{\Xb}}{1+\alpha}$, $\Deltab \triangleq \wt\Acal\rbr{\Xb} - \wt\Mb\Xb$, and its reweighted analog $\wt\Deltab \triangleq \wt\Mb \Xb \wt\Acal\rbr{\Xb}^\pinv \Deltab$, we further introduce positive semidefinite matrices: $\covs \triangleq \frac{\Sb^\top \Sb}{N}$, $\covaug \triangleq \frac{\Deltab^\top \Deltab}{(1+\alpha)N}$, and $\covaugwt \triangleq \frac{\wt\Deltab^\top \wt\Deltab}{(1+\alpha)N}$.
For demonstration purpose, we consider fixed $\Xb$ and $\wt\Acal\rbr{\Xb}$, with respect to which we introduce distortion factors $c_X, c_S>0$ as the minimum constants that satisfy $\covall \aleq c_X \covtr$ and $\covall \aleq c_S \covs$ (notice that such $c_X, c_S$ exist almost surely when $\Xb$ and $\wt\Acal\rbr{\Xb}$ are drawn from absolutely continuous marginal distributions). 

Recall $\dau = \rank\rbr{\Deltab}$ from \Cref{sec:linear_regression_label_invariant}, let $\projrg \triangleq \Deltab^\pinv \Deltab$ denote the rank-$\dau$ orthogonal projector onto $\range\rbr{\Deltab^\top}$.
Then, for $L(\thetab) = \frac{1}{N}\nbr{\Xb\thetab - \yb}_2^2$, we have the following result:

\begin{figure}
    \begin{subfigure}{0.45\columnwidth}
    \centering
	\includegraphics[width=\linewidth]{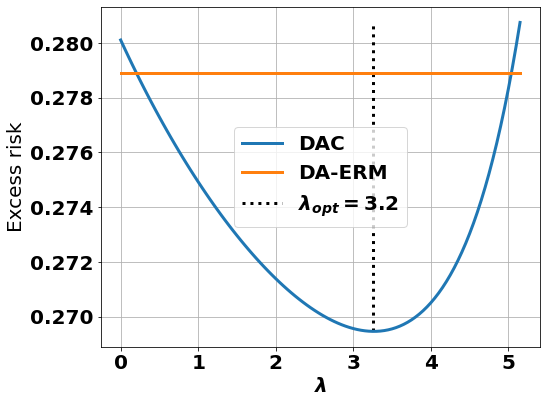}
	\vspace{-2em}
	\caption{Comparison of DAC with different $\lambda$ (optimal choice at $\lambda_{\textit{opt}}=3.2$) and DA-ERM in \Cref{example:misspec}, where $\dau=24$ and $\alpha=1$. The results demonstrate that, with a proper $\lambda$, DAC can outperform DA-ERM under misspecified augmentations.}
\label{fig:misspec_lambda}
    \end{subfigure}
    \hfill
    \begin{subfigure}{0.45\columnwidth}
    \centering
	\includegraphics[width=\linewidth]{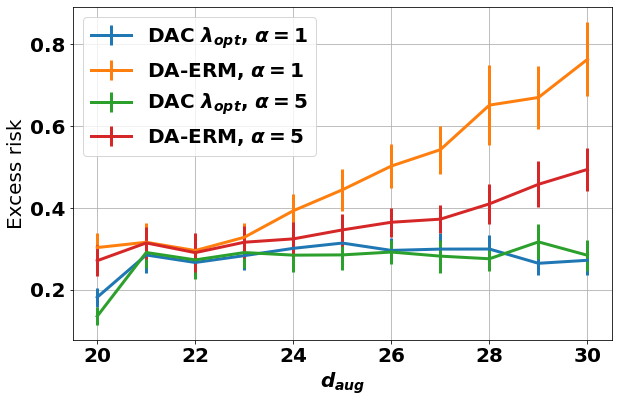}
	\vspace{-2em}
	\caption{Comparison of DAC with the optimal $\lambda$ and DA-ERM in \Cref{example:misspec} for different augmentation strength $\dau$. $\dau=20$ corresponds to the label-invariance augmentations, whereas increasing $\dau$ leads to more misspecification. 
	}
\label{fig:misspec_daug}
    \end{subfigure}
    \caption{Comparisons of DAC and DA-ERM under misspecification.}
\end{figure}

\begin{theorem}\label{thm:formal_linear_regression_soft}
Learning with DAC regularization (\Cref{eq:dac_soft_reg}), we have that $\EE_{\epsb}\sbr{L(\wh\thetab^{dac}) - L\rbr{\thetab^*}} \leq \frac{\sigma^2 \rbr{d-\dau}}{N} + \norm{\projrg \thetab^*}_{\covaug} \sqrt{\frac{\sigma^2}{N} \tr\rbr{\covtr \covaug^\pinv}}$ at the optimal $\lambda$, 
\footnote{A positive (semi)definite matrix $\Sigmab$ induces a (semi)norm: $\nbr{\ub}_\Sigmab = \rbr{\ub^\top \Sigmab \ub}^{1/2}$ for all conformable $\ub$.}
while learning with DA-ERM (\Cref{eq:plain_erm}), $\EE_{\epsb} \sbr{L(\wh\thetab^{\herm}) - L\rbr{\thetab^*}} \geq \frac{\sigma^2 d}{N c_X c_S} + \nbr{\projrg \thetab^*}_{\covaugwt}^2$.  Here, $\projrg \thetab^*$ measures the misspecification in $\thetab^*$ by the augmentations $\wt\Acal\rbr{\Xb}$.
\end{theorem}

One advantage of DAC regularization derives from its flexibility in choosing regularization parameter $\lambda$. With a proper $\lambda$ (\eg, see \Cref{fig:misspec_lambda}) that matches the misspecification $\constmis^2 = \frac{1}{\rbr{1+\alpha} N} \norm{\rbr{\wt\Acal\rbr{\Xb} - \wt\Mb\Xb} \thetab^*}^2_2 = \nbr{\projrg \thetab^*}^2_{\covaug}$, DAC effectively reduces the function class from $\R^d$ to $\csepp{\thetab}{\nbr{\projrg \thetab}_{\covaug} \le \constmis}$ and improves sample efficiency.

Another advantage of DAC is that, in contrast to DA-ERM, consistency regularization in \Cref{eq:dac_soft_reg} refrains from learning the original labels with misspecified augmentations. This allows DAC to learn from fewer but stronger (potentially more severely misspecified) augmentations (\eg, \Cref{fig:misspec_daug}). Specifically, as $N \to \infty$, the excess risk of DAC with the optimal $\lambda$ converges to zero by learning from unbiased labels $\E_{\epsb}\sbr{\yb} = \Xb\thetab^*$, whereas DA-ERM suffers from a bias term $\nbr{\projrg \thetab^*}_{\covaugwt}^2 > 0$ due to the bias from misspecified augmentations $\E_{\epsb}\sbr{\wt\Mb\yb} \neq \wt\Acal\rbr{\Xb}\thetab^*$. 

\begin{example}\label{example:misspec} 
As in \Cref{example:linear_regression}, we consider a linear regression problem of dimension $d=30$ with $\alpha \ge 1$ misspecified augmentations on $N=50$ $\iid$ training samples drawn from $\Ncal(\b0, \Ib_{d})$.  
We aim to learn $\thetab^* = [\thetab_c^*; \mathbf{0}] \in \RR^{d}$ (where $\thetab_c^* \in \cbr{-1,+1}^{d_c}$, $d_c=10$) under label noise $\sigma=0.1$. 
The misspecified augmentations mimic the effect of color jitter by adding $\iid$ Gaussian noise entry-wisely to the last $\dau$ feature coordinates: $\wt\Acal\rbr{\Xb} = \sbr{\Xb;\Xb'}$ where $\Xb'_{ij} = \Xb_{ij} + \Ncal\rbr{0, 0.1}$ for all $i \in [N]$, $d-\dau+1 \le j \le d$ -- such that $\dau=\rank\rbr{\Deltab}$ with probability $1$. The $(d-\dau+1),\dots,d_c$-th coordinates of $\thetab^*$ are misspecified by the augmentations.

As previously discussed on \Cref{thm:formal_linear_regression_soft}, DAC is more robust than DA-ERM to misspecified augmentations, and therefore can learn with fewer (smaller $\alpha$) and stronger (larger $\dau$) augmentations. In addition, DAC generally achieves better generalization than DA-ERM with limited samples.
\end{example}
\section{Beyond Linear Model}\label{sec:beyond_linear}

In this section, we extend our analysis of DAC regularization to non-linear models, including the two-layer neural networks, and DNN-based classifiers with expansion-based augmentations.  

Further, in addition to the popular in-distribution setting where we consider a unique distribution $\Pgt$ for both training and testing, DAC regularization is also known to improve out-of-distribution generalization for settings like domain adaptation. We defer detailed discussion on such advantage of DAC regularization for linear regression in the domain adaptation setting to \Cref{apx:case_ood}.

\subsection{Two-layer Neural Network} \label{sec:example_2relu}
We first generalize our analysis to an illustrative nonlinear model -- two-layer ReLU network.
With $\Xcal = \R^d$ and $\Ycal = \R$, we consider a ground truth distribution $\Pgt\rsep{y}{\xb}$ induced by $y = \rbr{\xb^{\top} \Bb^*}_+ \wb^* + \eps$. 
For the unknown ground truth function $h^*\rbr{\xb} \triangleq \rbr{\xb^{\top} \Bb^*}_+ \wb^*$, $(\cdot)_+ \triangleq \max(0,\cdot)$ denotes the element-wisely ReLU function; ${\Bb^*} = \bmat{\bb_1^* \dots \bb_k^* \dots \bb_q^*} \in \R^{d \times q}$ consists of $\bb_k^* \in \SSS^{d-1}$ for all $k \in [q]$; and $\eps \sim \Ncal\rbr{0, \sigma^2}$ is $\iid$ Gaussian noise.
In terms of the function class $\Hcal$, for some constant $C_w \geq \norm{\wb^*}_1$, let
\begin{align*}
    \Hcal = \csepp{h(\xb) = (\xb^{\top} \Bb)_+ \wb}
    {\Bb=[\bb_1 \dots \bb_q] \in \R^{d \times q}, \norm{\bb_k}_2=1\ \forall\ j \in [q], \norm{\wb}_1 \leq C_w},
\end{align*}
such that $h^* \in \Hcal$. For regression, we again consider square loss $l(h(\xb), y) = \frac{1}{2}(h(\xb)-y)^2$ and learn with DAC on the first layer: $\rbr{\xb_{i}^{\top} \Bb}_{+} = \rbr{\xb_{i,j}^{\top} \Bb}_{+}$.  

Let $\Deltab \triangleq \wt\Acal(\Xb)-\wt\Mb\Xb$, and $\projnull$ be the projector onto the null space of $\Deltab$. Under mild regularity conditions (\ie, $\alpha N$ being sufficiently large, $\xb$ being subgaussian, and distribution of $\Deltab$ being absolutely continuous, as specified in \Cref{apx:pf_case_2layer_relu}), regression over two-layer ReLU networks with the DAC regularization generalizes as following:
\begin{theorem}[Informal result on two-layer neural network with DAC (formally in \Cref{thm:case_2layer_relu_risk})]\label{thm:case_2layer_relu_risk_informal}
Conditioned on $\Xb$ and $\Deltab$, with $L(h) = \frac{1}{N}\nbr{h(\Xb) - h^*(\Xb)}_2^2$ and $\sqrt{\frac{1}{N}\sum_{i=1}^N \nbr{\projnull \xb_i}^2_2} \le \Cnull$, for any $\delta \in (0,1)$, with probability at least $1-\delta$ over $\epsb$, 
\begin{align*}
    L\rbr{\wh{h}^{dac}} - L\rbr{h^*}
    \lesssim 
    \sigma C_w \Cnull \rbr{\frac{1}{\sqrt{N}} + \sqrt{\frac{\log(1/\delta)}{N}}}.
\end{align*}
\end{theorem} 

Recall $\dau = \rank(\Deltab)$. With a sufficiently large $N$ (as specified in \Cref{apx:pf_case_2layer_relu}), we have $\Cnull \lesssim \sqrt{d-\dau}$ with high probability\footnote{Here we only account for the randomness in $\Xb$ but not that in $\Deltab|\Xb$ which characterizes $\dau$ for conciseness. We refer the readers to \Cref{apx:pf_case_2layer_relu} for a formal tail bound on $\Cnull$.}. Meanwhile, applying DA-ERM directly on the augmented samples achieves no better than $L(\wh h^{\herm}) - L(h^{*}) \lesssim \sigma C_w {\max\rbr{\sqrt{\frac{d}{(\alpha+1) N}}, \sqrt{\frac{d - \dau}{N}}}}$, where the first term corresponds to the generalization bound for a $d$-dimensional regression with $(\alpha+1) N$ samples, and the second term follows as the augmentations fail to perturb a $(d - \dau)$-dimensional subspace (and in which DA-ERM can only rely on the $N$ original samples for learning). In specific, the first term will dominate the $\max$ with limited augmented data (i.e., $\alpha$ being small). 

Comparing the two, we see that DAC tends to be more efficient than DA-ERM, and such advantage is enhanced with strong but limited data augmentations (i.e., large $\dau$ and small $\alpha$). For instance, with $\alpha = 1$ and $\dau = d - 1$, the generalization error of DA-ERM scales as $\sqrt{{d}/{N}}$, while DAC yields a dimension-free $\sqrt{{1}/{N}}$ error.

As a synopsis for the regression cases in \Cref{sec:linear_regression_label_invariant}, \Cref{subsec:finite_lambda}, and \Cref{sec:example_2relu} generally, the effect of DAC regularization can be casted as a dimension reduction by $\dau$ -- dimension of the subspace perturbed by data augmentations where features contain scarce label information.

\subsection{Classification with Expansion-based Augmentations}\label{subsec:expansion_based}

A natural generalization of the dimension reduction viewpoint on DAC regularization in the regression setting is the complexity reduction for general function classes. Here we demonstrate the power of DAC on function class reduction in a DNN-based classification setting.

Concretely, we consider a multi-class classification problem: given a probability space $\Xcal$ with marginal distribution $\Pgt(\xb)$ and $K$ classes $\Ycal=[K]$, let $h^*: \Xcal \to [K]$ be the ground truth classifier, partitioning $\Xcal$ into $K$ disjoint sets $\cbr{\Xcal_k}_{k \in [K]}$ such that $\Pgt\rbr{y|\xb} = \b{1}\cbr{y = h^*\rbr{\xb}} = \b{1}\cbr{\xb \in \Xcal_y}$.
In the classification setting, we concretize \Cref{def:causal_invar_data_aug} with \textit{expansion-based data augmentations} introduced in \cite{wei2021theoretical, cai2021theory}.

\begin{definition}[Expansion-based augmentations (formally in \Cref{def:generalized-causal-invariant-data-augmentation})]
\label{def:generalized-causal-invariant-data-augmentation_informal}
With respect to an augmentation function $\Acal:\Xcal \to 2^{\Xcal}$, let $\nbh(S) \triangleq \cup_{\xb \in S} \cbr{\xb' \in \Xcal ~\big|~ \Acal(\xb) \cap \Acal(\xb') \neq \emptyset}$ be the neighborhood of $S \subseteq \Xcal$.
For any $c>1$, we say that $\Acal$ induces $c$-expansion-based data augmentations if (a) $\cbr{\xb} \subsetneq \Acal(\xb) \subseteq \cbr{\xb' \in \Xcal ~|~ h^*(\xb) = h^*(\xb')}$ for all $\xb \in \Xcal$; and (b) for all $k \in [K]$, given any $S \subseteq \Xcal$ with $\Pgt\rbr{S \cap \Xcal_k} \leq \frac{1}{2}$, $\Pgt\rbr{\nbh\rbr{S} \cap \Xcal_k} \geq \min\cbr{c \cdot \Pgt\rbr{S \cap \Xcal_k},1}$.
\end{definition}

Particularly, \Cref{def:generalized-causal-invariant-data-augmentation_informal}(a) enforces that the ground truth classifier $h^*$ is invariant throughout each neighborhood. Meanwhile, the expansion factor $c$ in \Cref{def:generalized-causal-invariant-data-augmentation_informal}(b) serves as a quantification of augmentation strength -- a larger $c$ implies a stronger augmentation $\Acal$.

We aim to learn $h(\xb) \triangleq \argmax_{k \in [K]}\ f(\xb)_k$ with loss $l_{01}\rbr{h(\xb),y} = \b1\cbr{h(\xb) \neq y}$ from $\Hcal$ induced by the class of $p$-layer fully connected neural networks with maximum width $q$, $\Fcal = \csepp{f: \Xcal \to \R^K}{f = f_{2p-1} \circ \dots \circ f_1,}$ where $f_{2\iota-1}(\xb) = \Wb_{\iota} \xb,\ f_{2\iota}(\epsb)=\varphi(\epsb)$, $\Wb_{\iota} \in \R^{d_{\iota} \times d_{\iota-1}}$ $\forall \iota \in [p]$, $q \triangleq \max_{\iota\in[p]} d_{\iota}$, and $\varphi$ is the activation function. 

Over a general probability space $\Xcal$, DAC with expansion-based augmentations requires stronger conditions than merely consistent classification over $\Acal(\xb_i)$ for all labeled training samples $i \in [N]$. Instead, we enforce a large robust margin $m_{\Acal}(f,\xb^u)$ (adapted from \cite{wei2021theoretical}, see \Cref{apx:generalized_DAC}) over an finite set of unlabeled samples $\Xb^u$ that is independent of $\Xb$ and drawn $\iid$ from $P(\xb)$. Intuitively, $m_{\Acal}(f, \xb^u)$ measures the maximum allowed perturbation in all parameters of $f$ such that predictions remain consistent throughout $\Acal\rbr{\xb^u}$ ($\eg$, $m_{\Acal}(f,\xb^u) > 0$ is equivalent to enforcing consistent classification outputs).
For any $0< \tau \leq \max_{f \in \Fcal}\ \inf_{\xb^u \in \Xcal} m_{\Acal}(f, \xb^u)$, the DAC regularization reduces the function class $\Hcal$ to
\begin{align*}
    \Hred \triangleq \csepp{h \in \Hcal}{m_{\Acal}(f,\xb^u)> \tau \quad \forall\ \xb^u \in \Xb^u}
\end{align*}
such that for $\hgdacfin = \argmin_{h \in \Hred} \frac{1}{N} \sum_{i=1}^N l_{01}\rbr{h(\xb_i),y_i}$, we have the following.

\begin{theorem}[Informal result on classification with DAC (formally in \Cref{thm:generalized-dac-finite-unlabeled})]
\label{thm:generalized-dac-finite-unlabeled_informal}
Let $\mu \triangleq \sup_{h \in \Hred} \PP_{\Pgt} \sbr{\exists\ \xb' \in \Acal(\xb): h(\xb) \neq h(\xb')} \leq \frac{c-1}{4}$. For any $\delta \in (0,1)$, with probability at least $1-\delta$, we have $\mu \leq \wt O \rbr{\frac{\sum_{\iota=1}^p \sqrt{q} \norm{\Wb_{\iota}}_F}{\tau \sqrt{\abbr{\Xb^u}}} + \sqrt{\frac{p \log \abbr{\Xb^u}}{\abbr{\Xb^u}}}}$ such that
\begin{align*}
    L_{01}\rbr{\hgdacfin} - L_{01}\rbr{h^*} 
    \lesssim & \sqrt{\frac{K \log K}{N} + \frac{K \mu}{\min\cbr{c-1,1}}} + \sqrt{\frac{\log(1/\delta)}{N}}.
\end{align*}
\end{theorem}  

In particular, DAC regularization leverages the unlabeled samples $\Xb^u$ and effectively decouples the labeled sample complexity $N = O\rbr{K \log K}$ from complexity of the function class $\Hcal$ (characterized by $\cbr{\Wb_\iota}_{\iota \in [p]}$ and $q$ and encapsulated in $\mu$) via the reduced function class $\Hred$. Notably, \Cref{thm:generalized-dac-finite-unlabeled_informal} is reminiscent of \cite{wei2021theoretical} Theorem 3.6, 3.7, and \cite{cai2021theory} Theorem 2.1, 2.2, 2.3. We unified the existing theories under our function class reduction viewpoint to demonstrate its generality.

\section{Experiments}\label{sec:experiments}

In this section, we empirically verify that training with DAC learns more efficiently than DA-ERM. The dataset is derived from CIFAR-100, where we randomly select 10,000 labeled data as the training set (i.e., 100 labeled samples per class). During the training time, given a training batch, we generate augmentations by RandAugment \citep{cubuk2020randaugment}. We set the number of augmentations per sample to 7 unless otherwise mentioned.

The experiments focus on comparisons of 1) training with consistency regularization (DAC), and 2) empirical risk minimization on the augmented dataset (DA-ERM). We use the same network architecture (a WideResNet-28-2 \citep{zagoruyko2016wide}) and the same training settings (e.g., optimizer, learning rate schedule, etc) for both methods. We defer the detailed experiment settings to \Cref{apdx:exp_detail}. Our test set is the standard CIFAR-100 test set, and we report the average and standard deviation of the testing accuracy of 5 independent runs. The consistency regularizer is implemented as the $l_2$ distance of the model's predictions on the original and augmented samples.

\textbf{Efficacy of DAC regularization.} We first show that the DAC regularization learns more efficiently than DA-ERM. The results are listed in \Cref{table:different_lambda}. In practice, the augmentations almost always alter the label distribution, we therefore follow the discussion in \cref{subsec:finite_lambda} and adopt a finite $\lambda$ (i.e., the multiplicative coefficient before the DAC regularization, see \Cref{eq:dac_soft}). With proper choice of $\lambda$, training with DAC significantly improves over DA-ERM.

\begin{table}[h]
\centering
\begin{tabular}{c|ccccc}
\hline
\multirow{2}{*}{DA-ERM} & \multicolumn{5}{c}{DAC Regularization}                                              \\
                                       & $\lambda=0$ & $\lambda=1$ & $\lambda=5$ & $\lambda=10$ & $\lambda=20$            \\ \hline
$69.40 \pm 0.05$                                  & $62.82 \pm 0.21$       & $68.63 \pm 0.11$       & $\mathbf{70.56 \pm 0.07}$       & $\mathbf{70.52 \pm 0.14}$        & $68.65 \pm 0.27$     \\ \hline
\end{tabular}
\caption{Testing accuracy of DA-ERM and DAC with different $\lambda$'s (regularization coeff.).}
\label{table:different_lambda}
\end{table}

\textbf{DAC regularization helps more with limited augmentations.} Our theoretical results suggest that the DAC regularization learns efficiently with a limited number of augmentations. While keeping the number of labeled samples to be 10,000, we evaluate the performance of the DAC regularization and DA-ERM with different numbers of augmentations. The number of augmentations for each training sample ranges from 1 to 15, and the results are listed in \Cref{table:different_number_of_aug}. The DAC regularization offers a more significant improvement when the number of augmentations is small. This clearly demonstrates that the DAC regularization learns more efficiently than DA-ERM.

\begin{table}[h]
\centering
\begin{tabular}{c|cccc}
\hline
Number of Augmentations & 1 & 3 & 7 & 15 \\ \hline
DA-ERM                     & $67.92 \pm 0.08$  & $69.04 \pm 0.05$  & $69.25 \pm 0.16$  &  $69.30 \pm 0.11$  \\
DAC ($\lambda=10$)       & $\mathbf{70.06 \pm 0.08}$  & $\mathbf{70.77 \pm 0.20}$  & $\mathbf{70.74 \pm 0.11}$  &  $\mathbf{70.31 \pm 0.12}$  \\ \hline
\end{tabular}
\caption{Testing accuracy of DA-ERM and DAC with different numbers of augmentations.}
\label{table:different_number_of_aug}
\end{table}

\begin{wrapfigure}{r}{0.4\columnwidth}
\vspace{-.5em}
	\centering
	\includegraphics[clip, trim={180 30 280 20}, width=\linewidth]{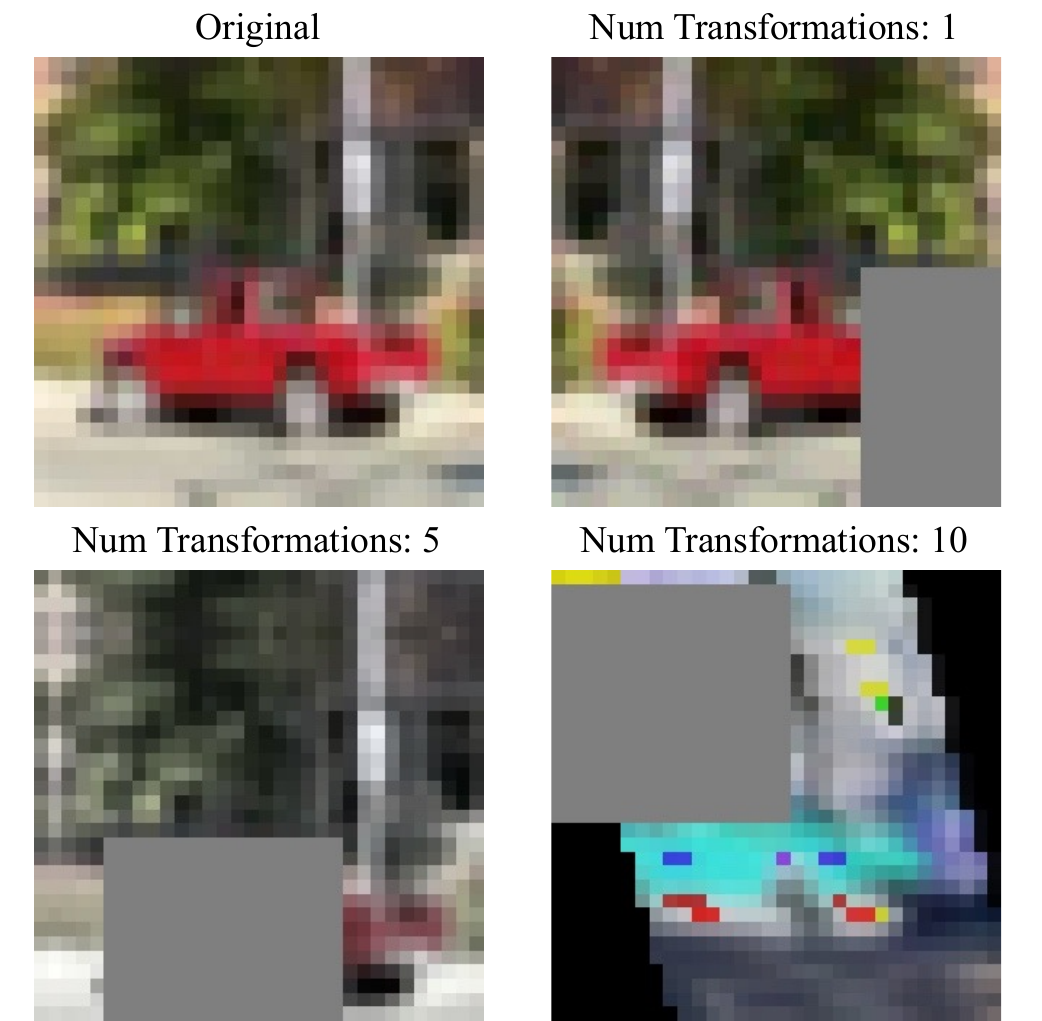}
	\vspace{-1em}
	\caption{Examples of different numbers of transformations.}
\label{fig:augmentation_strength}
\vspace{-1em}
\end{wrapfigure}

\textbf{Proper augmentation brings good performance.} Proper data augmentation is important - it needs to well perturb the input features, but it should also leave the label distribution unchanged. Here we experiment with different numbers of random transformations (e.g., random cropping, flipping, etc.) applied to the training samples sequentially. More transformations perturb the input more, but it is more likely to alter the label distribution. The number of transformations ranges from 1 to 10, and the results are listed in \Cref{table:different_aug_strength}. We see that both DA-ERM and the DAC regularization benefit from a proper augmentation. When too many transformations are applied (e.g., Num Transformations: 10, as shown in \Cref{fig:augmentation_strength}), the DAC regularization gives a worse performance. We believe the reason is that DAC falsely enforces consistency regularization where the labels have changed.
\begin{table}[h]
\centering
\begin{tabular}{c|cccc}
\hline
Num Transformations & 1 & 2 & 5 & 10 \\ \hline
DA-ERM                   & $68.56 \pm 0.12$ & $69.32 \pm 0.11$  & $\mathbf{69.97 \pm 0.14}$  & $\mathbf{69.66 \pm 0.16}$ \\
DAC ($\lambda=10$)     & $\mathbf{70.66 \pm 0.14}$  & $\mathbf{70.65 \pm 0.07}$  & $\mathbf{70.01 \pm 0.10}$  &  $68.95 \pm 0.27$ \\ \hline
\end{tabular}
\caption{Testing accuracy of DA-ERM and DAC with various numbers of transformations.}
\label{table:different_aug_strength}
\end{table}

\begin{wraptable}{hr}{0.65\columnwidth}
% \vspace{-1em}
\centering
\begin{tabular}{c|ccc}
\hline
Number of Unlabeled Data     & 5000 & 10000 & 20000 \\ \hline
FixMatch                     &  67.74    &  69.23     & 70.76          \\
FixMatch + DAC ($\lambda=1$) &  \textbf{71.24}    &   \textbf{72.7}    &  \textbf{74.04}     \\ \hline
\end{tabular}
\caption{DAC helps FixMatch when the unlabeled data is scarce.}
\vspace{-0.5em}
\label{table:combining_with_SSL}
\end{wraptable}

\textbf{Combining with a semi-supervised learning algorithm.} Here we show that the DAC regularization can be easily extended to the semi-supervised learning setting. We take the previously established semi-supervised learning method FixMatch \citep{sohn2020fixmatch} as the baseline and adapt the FixMatch by combining it with the DAC regularization. Specifically, besides using FixMatch to learn from the unlabeled data, we additionally generate augmentations for the labeled samples and apply DAC. In particular, we focus on the data-scarce regime by only keeping 10,000 labeled samples and at most 20,000 unlabeled samples. Results are listed in \Cref{table:combining_with_SSL}. We see that the DAC regularization also improves the performance of FixMatch when the unlabeled samples are scarce. This again demonstrates the efficiency of learning with DAC.
\section{Conclusion}

In this paper, we take a step toward understanding the statistical efficiency of DAC with limited data augmentations. At the core, DAC is statistically more efficient because it reduces problem dimensions by enforcing consistency regularization. 

We demonstrate the benefits of DAC compared to DA-ERM (expanding training set with augmented samples) both theoretically and empirically. Theoretically, we show a strictly smaller generalization error under linear regression, and explicitly characterize the generalization upper bound for two-layer neural networks and expansion-based data augmentations. We further show that DAC better handles the label misspecification caused by strong augmentations. Empirically, we provide apples-to-apples comparisons between DAC and DA-ERM. These together demonstrate the superior efficacy of DAC over DA-ERM.

\paragraph{Acknowledgement}
SY's research is supported by NSF grants 1564000 and 1934932. YD's research is supported by AFOSR MURI FA9550-19-1-0005, NSF DMS 1952735, NSF HDR-1934932, and NSF 2019844.

\bibliography{ref}
\bibliographystyle{abbrvnat}

\newpage

\appendix
\section{Linear Regression Models}\label{apdx:general}

In this section, we present formal proofs for the results on linear regression in the fixed design where the training samples $\rbr{\Xb,\yb}$ and their augmentations $\wt\Acal\rbr{\Xb}$ are considered to be fixed. We discuss two types of augmentations: the label invariant augmentations in \Cref{sec:linear_regression_label_invariant} and the misspecified augmentations in \Cref{subsec:finite_lambda}.

\subsection{Linear Regression with Label Invariant Augmentations}\label{apdx:linear_regression}
 
For fixed $\wt\Acal(\Xb)$, let $\Deltab \triangleq \widetilde \Acal(\Xb) - \widetilde \Mb \Xb$ in this section. We recall that $\dau = \rank\rbr{\Deltab}$ since there is no randomness in $\widetilde \Acal, \Xb$ in fix design setting. Assuming that $\widetilde \Acal(\Xb)$ admits full column rank, we have the following theorem on the excess risk of DAC and ERM:
\begin{theorem}[Formal restatement of \Cref{thm:informal_linear_regression} on linear regression.]\label{thm:formal_linear_regression}
Learning with DAC regularization, we have $\EE\sbr{L(\widehat \thetab^{dac}) - L(\thetab^*)} = \frac{(d - \dau)\sigma^2}{N}$, while learning with ERM directly on the augmented dataset, we have $\EE\sbr{L(\widehat \thetab^{\herm}) - L(\thetab^*)} = \frac{(d - \dau + d')\sigma^2}{N}$. $d'$ is defined as
\begin{align*}
    d' \triangleq \frac{\tr\rbr{\widetilde\Mb^\top\rbr{\projAX - \Pb_{\Scal}}\widetilde\Mb}}{1+\alpha},
\end{align*}
where $d' \in [0, \dau]$ with $\projAX = \widetilde \Acal(\Xb)\rbr{\widetilde \Acal(\Xb)^\top\widetilde \Acal(\Xb)}^{-1}\widetilde \Acal(\Xb)^\top$ and $\Pb_\Scal \in \RR^{(\alpha+1)N \times (\alpha+1)N}$ is the orthogonal projector onto $\Scal \triangleq \cbr{\widetilde \Mb \Xb \thetab~|~\forall \thetab \in \RR^d, s.t. \rbr{\widetilde\Acal(\Xb) - \widetilde \Mb \Xb}\thetab = \b0}$.
\end{theorem}

\begin{proof}
With $L(\thetab) \triangleq \frac{1}{(1+\alpha) N}\norm{\widetilde \Acal(\Xb)\thetab - \widetilde \Acal(\Xb)\thetab^*}_2^2$, the excess risk of ERM on the augmented training set satisfies that:
\begin{align*}
    \EE\sbr{L(\widehat \thetab^{\herm})} &= \frac{1}{(1 + \alpha) N}\EE\sbr{\norm{\widetilde \Acal(\Xb)\widehat\thetab^{\herm} - \widetilde \Acal(\Xb)\thetab^{*}}_{2}^{2}} \\
    &= \frac{1}{(1 + \alpha) N}\EE\sbr{\norm{\widetilde \Acal(\Xb)(\widetilde \Acal(\Xb)^{\top}\widetilde \Acal(\Xb))^{-1}\widetilde \Acal(\Xb)^{\top}(\widetilde \Acal(\Xb)\thetab^{*} + \widetilde \Mb\epsb) - \widetilde \Acal(\Xb)\thetab^{*}}_{2}^{2}} \\
    &=\frac{1}{(1 + \alpha) N}\EE\sbr{\norm{\projAX\widetilde\Acal(\Xb)\thetab^{*} + \projAX \widetilde \Mb\epsb - \widetilde \Acal(\Xb)\thetab^{*}}_{2}^{2}} \\
    &=\frac{1}{(1 + \alpha) N}\EE\sbr{\norm{\projAX \widetilde \Mb\epsb}_{2}^{2}} \\
    &=\frac{1}{(1 + \alpha) N}\EE\sbr{\tr(\epsb^{\top} \widetilde \Mb^{\top}\projAX \widetilde \Mb\epsb)} \\
    &=\frac{\sigma^{2}}{(1 + \alpha) N} \tr\rbr{\wt\Mb^{\top} \projAX \wt\Mb}
.\end{align*}
Let $\Ccal_{\widetilde \Acal(\Xb)}$ and $\Ccal_{\widetilde \Mb}$ denote the column space of $\widetilde \Acal(\Xb)$ and $\widetilde \Mb$, respectively. Notice that $\Scal$ is a subspace of both $\Ccal_{\widetilde \Acal(\Xb)}$ and $\Ccal_{\widetilde \Mb}$. Observing that $\dau = \rank\rbr{\Deltab} = \rank\rbr{\Pb_{\Scal}}$, we have 
\begin{align*}
\EE\sbr{L(\widehat \thetab^{\herm})}
=&\frac{\sigma^{2}}{(1 + \alpha) N} \tr(\widetilde\Mb^{\top}\projAX\widetilde\Mb)
\\
= &\frac{\sigma^{2}}{(1 + \alpha) N} \tr(\widetilde\Mb^{\top}\Pb_\Scal\widetilde\Mb) + \frac{\sigma^{2}}{(1+\alpha)N} \tr(\widetilde\Mb^{\top}(\projAX - \Pb_\Scal)\widetilde\Mb)
\\
= &\frac{\sigma^{2}}{(1 + \alpha) N} \tr(\widetilde\Mb^{\top}\Pb_\Scal\widetilde\Mb) + \frac{\sigma^{2}}{N}\cdot \frac{\tr(\widetilde\Mb^{\top}(\projAX -\Pb_\Scal)\widetilde\Mb)}{1+\alpha}
\end{align*}

By the data augmentation consistency constraint, we are essentially solving the linear regression on the $(d-\dau)$-dimensional space $\cbr{\thetab~|~\Deltab \thetab = 0}$. The rest of proof is identical to standard regression analysis, with features first projected to $\Scal$: 
\begin{align*}
\EE\sbr{L(\widehat \thetab^{dac})} &= \frac{1}{(1 + \alpha)N}\EE\sbr{\norm{\widetilde \Acal(\Xb)\widehat\thetab^{dac} - \widetilde \Acal(\Xb)\thetab^{*}}_{2}^{2}} 
\\
&= \frac{1}{(1 + \alpha) N}\EE\sbr{\norm{\widetilde \Acal(\Xb)(\widetilde \Acal(\Xb)^{\top}\widetilde \Acal(\Xb))^{-1}\widetilde \Acal(\Xb)^{\top}\Pb_\Scal(\widetilde \Acal(\Xb)\thetab^{*} + \widetilde \Mb\epsb) - \widetilde \Acal(\Xb)\thetab^{*}}_{2}^{2}} 
\\
&=\frac{1}{(1 + \alpha) N}\EE\sbr{\norm{ \projAX \Pb_\Scal \widetilde\Acal(\Xb)\thetab^{*} + \projAX \Pb_\Scal \widetilde \Mb\epsb - \widetilde \Acal(\Xb)\thetab^{*}}_{2}^{2}} 
\\
& \quad \rbr{\t{since}\ \wt\Acal(\Xb) \thetab^* \in \Scal,\ \t{and}\ \projAX \Pb_\Scal = \Pb_\Scal\ \t{since}\ \Scal \subseteq \Ccal_{\wt\Acal(\Xb)}}
\\
&=\frac{1}{(1 + \alpha)N}\EE\sbr{\norm{\Pb_\Scal \widetilde \Mb\epsb}_{2}^{2}} \\
&=\frac{\sigma^{2}}{(1 + \alpha) N} \tr(\widetilde\Mb^{\top}\Pb_\Scal\widetilde\Mb) \\
&= \frac{(d - \dau)\sigma^{2}}{N}.
\end{align*}

\end{proof}

\subsection{Linear Regression Beyond Label Invariant Augmentations}\label{apx:finite_lambda}

\begin{proof}[Proof of \Cref{thm:formal_linear_regression_soft}]
With $L(\thetab) \triangleq \frac{1}{N}\norm{\Xb\thetab - \Xb\thetab^*}_2^2 = \nbr{\thetab - \thetab^*}_{\covtr}^2$, we start by partitioning the excess risk into two parts -- the variance from label noise and the bias from feature-label mismatch due to augmentations ($\ie$, $\wt\Acal\rbr{\Xb}\thetab^* \neq \wt\Mb\Xb\thetab^*$): 
\begin{align*}
    \E_{\epsb}\sbr{L\rbr{\thetab} - L\rbr{\thetab^*}} 
    = \E_{\epsb}\sbr{\norm{\thetab - \thetab^*}_{\covtr}^2} 
    = \underbrace{\E_{\epsb}\sbr{\nbr{\thetab - \E_{\epsb}\sbr{\thetab}}_{\covtr}^2}}_{\t{Variance}} + 
    \underbrace{\nbr{\E_{\epsb}\sbr{\thetab} - \thetab^*}_{\covtr}^2}_{\t{Bias}}
.\end{align*}

First, we consider learning with DAC regularization with some finite $0<\lambda<\infty$,
\begin{align*}
    \wh\thetab^{dac} = \argmin_{\thetab \in \R^d} \frac{1}{N} \nbr{\Xb \thetab - \yb}_2^2
    + \frac{\lambda}{\rbr{1+\alpha} N} \norm{\rbr{\wt\Acal\rbr{\Xb} - \wt\Mb\Xb} \thetab}_2^2.
\end{align*}
By setting the gradient of \Cref{eq:dac_soft_reg} with respect to $\thetab$ to $\b0$, with $\yb = \Xb \thetab^* + \epsb$, we have
\begin{align*}
    \wh\thetab^{dac} = \frac{1}{N} \rbr{\covtr + \lambda \covaug}^{\pinv} \Xb^\top \rbr{\Xb \thetab^* + \epsb},
\end{align*}    
Then with $\E_{\epsb}\sbr{\wh\thetab^{dac}} = \rbr{\covtr + \lambda \covaug}^{\pinv} \covtr \thetab^*$, 
\begin{align*}
    \t{Var} = \E_{\epsb}\sbr{\nbr{\frac{1}{N} \rbr{\covtr + \lambda \covaug}^{\pinv} \Xb^\top \epsb}_{\covtr}^2},
    \quad
    \t{Bias} = \nbr{\rbr{\covtr + \lambda \covaug}^{\pinv} \covtr \thetab^* - \thetab^*}_{\covtr}^2.
\end{align*}

For the variance term, we have
\begin{align*}
    \t{Var} 
    = &\frac{\sigma^2}{N} \tr\rbr{\rbr{\covtr + \lambda \covaug}^{\pinv} \covtr \rbr{\covtr + \lambda \covaug}^{\pinv} \covtr}
    \\
    = &\frac{\sigma^2}{N} \tr\rbr{\sbr{\covtr^{1/2} \rbr{\covtr + \lambda \covaug}^{\pinv} \covtr^{1/2}}^2}
    \\
    = &\frac{\sigma^2}{N} \tr\rbr{ \rbr{\Ib_d + \lambda \covtr^{-1/2} \covaug \covtr^{-1/2}}^{-2} }
\end{align*}
For the semi-positive definite matrix $\covtr^{-1/2} \covaug \covtr^{-1/2}$, we introduce the spectral decomposition:
\begin{align*}
    \covtr^{-1/2} \covaug \covtr^{-1/2} = \underset{d \times \dau}{\Qb}\ \underset{\dau \times \dau}{\Gammab}\ \Qb^\top, 
    \quad
    \Gammab = \diag\rbr{\gamma_1,\dots,\gamma_{\dau}},
\end{align*}
where $\Qb$ consists of orthonormal columns and $\gamma_1 \geq \dots \geq \gamma_{\dau} > 0$. Then
\begin{align*}
    \t{Var}
    = \frac{\sigma^2}{N} \tr\rbr{\rbr{\Ib_d - \Qb\Qb^\top} + \Qb \rbr{\Ib_{\dau} + \lambda\Gammab}^{-2} \Qb^\top}
    = \frac{\sigma^2 \rbr{d-\dau}}{N} + \frac{\sigma^2}{N} \sum_{i=1}^{\dau} \frac{1}{\rbr{1+\lambda \gamma_i}^2}.
\end{align*}    

For the bias term, we observe that
\begin{align*}
    \t{Bias} 
    = &\nbr{\rbr{\covtr + \lambda \covaug}^{\pinv} \covtr \thetab^* - \thetab^*}_{\covtr}^2
    \\
    = &\nbr{\rbr{\covtr + \lambda \covaug}^{\pinv} \rbr{-\lambda\covaug} \thetab^*}_{\covtr}^2
    \\
    = &\nbr{ \rbr{\Ib_d + \lambda \covtr^{-\frac{1}{2}} \covaug \covtr^{-\frac{1}{2}}}^{-1} \rbr{\lambda \covtr^{-\frac{1}{2}} \covaug \covtr^{-\frac{1}{2}}} \rbr{\covtr^{1/2} \projrg \thetab^*} }_2^2.
\end{align*}
Then with $\vthetab \triangleq \covtr^{1/2} \projrg \thetab^*$, we have
\begin{align*}
    \t{Bias} 
    = \sum_{i=1}^{\dau} \vartheta_i^2 \rbr{\frac{\lambda \gamma_i}{1 + \lambda \gamma_i}}^2
\end{align*}

To simply the optimization of regularization parameter $\lambda$, we leverage upper bounds of the variance and bias terms:
\begin{align*}
    & \t{Var} - \frac{\sigma^2 \rbr{d-\dau}}{N} 
    \leq \frac{\sigma^2}{N} \sum_{i=1}^{\dau} \frac{1}{\rbr{1+\lambda \gamma_i}^2} 
    \leq \frac{\sigma^2}{2 N \lambda} \sum_{i=1}^{\dau} \frac{1}{\gamma_i} 
    \leq \frac{\sigma^2}{2 N \lambda} \tr\rbr{\covtr \covaug^\pinv},
    \\
    & \t{Bias} 
    = \sum_{i=1}^{\dau} \vartheta_i^2 \rbr{\frac{\lambda \gamma_i}{1 + \lambda \gamma_i}}^2
    \leq \frac{\lambda}{2} \sum_{i=1}^{\dau} \vartheta_i^2 \gamma_i
    = \frac{\lambda}{2} \nbr{\projrg \thetab^*}^2_{\covaug}
.\end{align*}
Then with $\lambda = \sqrt{ \frac{\sigma^2 \tr\rbr{\covtr \covaug^\pinv}}{N \norm{\projrg \thetab^*}_{\covaug}^{2}} }$, we have the generalization bound for $\wh\thetab^{dac}$ in \Cref{thm:formal_linear_regression_soft}.

Second, we consider learning with DA-ERM:
\begin{align*}
    \wh\thetab^{\herm} = \argmin_{\thetab \in \R^d} \frac{1}{\rbr{1+\alpha}N} \nbr{\wt\Acal\rbr{\Xb} \thetab - \wt\Mb \yb}_2^2.
\end{align*} 
With 
\begin{align*}
    \wh\thetab^{\herm} = \frac{1}{(1+\alpha)N} \covall^{-1} \wt\Acal\rbr{\Xb}^\top \wt\Mb \rbr{\Xb \thetab^* + \epsb}
,\end{align*}  
we again partition the excess risk into the variance and bias terms.
For the variance term, with the assumptions $\covall \aleq c_X \covtr$ and $\covall \aleq c_S \covs$, we have
\begin{align*}
    \t{Var} 
    = & \E_{\epsb}\sbr{\nbr{\frac{1}{(1+\alpha)N} \covall^{-1} \wt\Acal\rbr{\Xb}^\top \wt\Mb \epsb}_{\covtr}^2}
    \\
    = & \E_{\epsb}\sbr{\nbr{\frac{1}{N} \covall^{-1} \Sb^\top \epsb}_{\covtr}^2}
    \\
    = & \frac{\sigma^2}{N} \tr\rbr{\covtr \covall^{-1} \covs \covall^{-1}}
    \\
    \geq & \frac{\sigma^2}{N} \tr\rbr{\frac{1}{c_X c_S} \Ib_d}
    = \frac{\sigma^2 d}{N c_X c_S}
.\end{align*}

Additionally, for the bias term, we have
\begin{align*}
    \t{Bias} 
    = & \nbr{\frac{1}{(1+\alpha)N} \covall^{-1} \wt\Acal\rbr{\Xb}^\top \wt\Mb\Xb \thetab^* - \thetab^*}_{\covtr}^2
    \\
    = & \nbr{\rbr{\wt\Acal\rbr{\Xb}^\top \wt\Acal\rbr{\Xb}}^{-1} \wt\Acal\rbr{\Xb}^\top \Deltab \rbr{\projrg \thetab^*}}^2_{\covtr}
    \\
    = & \nbr{\wt\Acal\rbr{\Xb}^\pinv \Deltab \rbr{\projrg \thetab^*}}^2_{\covtr}
    = \nbr{\projrg \thetab^*}_{\covaugwt}^2
.\end{align*}
Combining the variance and bias leads to the generalization bound for $\wh\thetab^{\herm}$ in \Cref{thm:formal_linear_regression_soft}.
\end{proof}

\section{Two-layer Neural Network Regression}\label{apx:pf_case_2layer_relu}

In the two-layer neural network regression setting with $\Xcal = \R^d$ described in \Cref{sec:example_2relu}, let $\Xb \sim P^N(\xb)$ be a set of $N$ $\iid$ samples drawn from the marginal distribution $P(\xb)$ that satifies the following.
\begin{assumption}[Regularity of marginal distribution]
\label{ass:observable_marginal_distribution}
Let $\xb \sim \Pgt(\xb)$ be zero-mean $\E[\xb]=\b{0}$, with covairance matrix $\E[\xb \xb^{\top}]=\Sigmab_{\xb} \succ 0$ whose eigenvalues are bounded by constant factors $\Omega(1)=\sigma_{\min}(\Sigmab_\xb) \le \sigma_{\max}(\Sigmab_\xb) = O(1)$, such that $(\Sigmab_{\xb}^{-1/2} \xb)$ is $\rho^2$-subgaussian
\footnote{A random vector $\vb \in \R^d$ is $\rho^2$-subgaussian if for any unit vector $\ub \in \mathbb{S}^{d-1}$, $\ub^{\top} \vb$ is $\rho^2$-subgaussian, $\E \sbr{\exp(s \cdot \ub^{\top} \vb)} \leq \exp\rbr{s^2 \rho^2/2}$.}.
\end{assumption}    

For the sake of analysis, we isolate the augmented part in  $\wt\Acal(\Xb)$ and denote the set of these augmentations as
\begin{align*}
    \Acal(\Xb) = \sbr{\xb_{1,1}; \cdots; \xb_{N,1}; \cdots; \xb_{1, \alpha}; \cdots; \xb_{N, \alpha}} \in \Xcal^{\alpha N},
\end{align*}
where for each sample $i \in [N]$, $\cbr{\xb_{i,j}}_{j \in [\alpha]}$ is a set of $\alpha$ augmentations generated from $\xb_i$, and $\Mb \in \R^{\alpha N \times N}$ is the vertical stack of $\alpha$ $N \times N$ identity matrices.
Analogous to the notions with respect to $\wt\Acal(\Xb)$ in the linear regression cases in \Cref{apdx:general}, in this section, we denote $\Deltab \triangleq \Acal(\Xb) - \Mb\Xb$ and quantify the augmentation strength as
\begin{align*}
    \dau \triangleq \rank\rbr{\Deltab} = \rank\rbr{\wt\Acal\rbr{\Xb} - \wt\Mb\Xb}
\end{align*}
such that $0 \leq \dau \leq \min\rbr{d, \alpha N}$ can be intuitively interpreted as the number of dimensions in the span of the unlabeled samples, $\row(\Xb)$, perturbed by the augmentations.

Then, to learn the ground truth distribution $\yb = h^*(\Xb) + \epsb = \rbr{\Xb\Bb^*}_+ \wb^* + \epsb$ where $\epsb \sim \Ncal(\b0, \sigma^2 \Ib_N)$, training with the DAC regularization can be formulated explicitly as
\begin{align*}
    \wh{\Bb}^{dac}, \wh{\wb}^{dac} ~=~
    &\underset{\Bb \in \R^{d \times q}, \wb \in \R^q}{\argmin}\ \frac{1}{N} \norm{\yb - \rbr{\Xb\Bb}_+ \wb}_2^2 \\
    &\t{s.t.} \quad
    \Bb = \bmat{\bb_1 \dots \bb_k \dots \bb_q},
    \ \bb_k \in \mathbb{S}^{d-1}\ \forall\ k \in [q],
    \quad
    \norm{\wb}_1 \leq C_w \\
    &\rbr{\Aemp\rbr{\Xb} \Bb}_+ = \rbr{\Mb\Xb\Bb}_+.
\end{align*}
For the resulted minimizer $\wh{h}^{dac}(\xb) \triangleq (\xb^{\top} \wh\Bb^{dac})_+ \wh\wb^{dac}$, we have the following.
\begin{theorem}[Formal restatement of \Cref{thm:case_2layer_relu_risk_informal} on two-layer neural network with DAC]
\label{thm:case_2layer_relu_risk}
Under \Cref{ass:observable_marginal_distribution}, we suppose $\Xb$ and $\Deltab$ satisfy that
(a) $\alpha N \geq 4 \dau$; and
(b) $\Deltab$ admits an absolutely continuous distribution. Then conditioned on $\Xb$ and $\Deltab$, with $L(h) = \frac{1}{N}\nbr{h(\Xb) - h^*(\Xb)}_2^2$ and $\frac{1}{N}\sum_{i=1}^N \nbr{\projnull \xb_i}^2_2 \leq \Cnull^2$ for some $\Cnull>0$, for all $\delta \in (0,1)$, with probability at least $1-\delta$ (over $\epsb$), 
\begin{align*}
    L\rbr{\wh{h}^{dac}} - L\rbr{h^*}
    \lesssim 
    \sigma C_w \Cnull \rbr{\frac{1}{\sqrt{N}} + \sqrt{\frac{\log(1/\delta)}{N}}}.
\end{align*}
\end{theorem}

Moreover, to account for randomness in $\Xb$ and $\Deltab$, we introduce the following notion of augmentation strength.
\begin{definition}[Augmentation strength]\label{def:daug} 
For any $\delta \in [0,1)$, let
\begin{align*}
    \dau(\delta) \triangleq \argmax_{d'}\ \PP_{\Deltab} \sbr{\rank \rbr{\Deltab}<d'} \le \delta.
\end{align*}
\end{definition}
Intuitively, the \textit{augmentation strength} $\dau$ ensures that the feature subspace perturbed by the augmentations in $\Aemp(\Xb)$ has a minimum dimension $\dau(\delta)$ with probability at least $1 - \delta$. A larger $\dau(\delta)$ corresponds to stronger augmentations. For instance, when $\Aemp(\Xb)=\Mb\Xb$ almost surely ($\eg$, when the augmentations are identical copies of the original samples, corresponding to the weakest augmentation -- no augmentations at all), we have $\dau(\delta) = \dau = 0$ for all $\delta<1$. Whereas for randomly generated augmentations, $\dau$ is likely to be larger (i.e., with more dimensions being perturbed). For example in \Cref{example:misspec}, for a given $\dau$, with random augmentations $\Aemp\rbr{\Xb} = \Xb'$ where $\Xb'_{ij} = \Xb_{ij} + \Ncal\rbr{0, 0.1}$ for all $i \in [N]$, $d-\dau+1 \le j \le d$, we have $\rank\rbr{\Deltab}=\dau$ with probability $1$. That is $\dau(\delta)=\dau$ for all $\delta \ge 0$.

Leveraging the notion of augmentation strength in \Cref{def:daug}, we show that the stronger augmentations lead to the better generalization by reducing $\Cnull$ in \Cref{thm:case_2layer_relu_risk}.
\begin{corollary}\label{coro:two_layer_daug_bound}
When $N \gg \rho^4 d$ and $\alpha N \ge d$, for any $\delta \in (0,1)$, with probability at least $1-\delta$ (over $\Xb$ and $\Deltab$),
we have $\Cnull \lesssim \sqrt{d - \dau(\delta)}$.
\end{corollary}

To prove \Cref{thm:case_2layer_relu_risk}, we start by showing that, with sufficient samples ($\alpha N \ge 4 \dau$), consistency of the first layer outputs over the samples implies consistency of those over the population.
\begin{lemma}
\label{lemma:total_invertible_meas_zero}
Under the assumptions in \Cref{thm:case_2layer_relu_risk}, every size-$\dau$ subset of rows in $\Deltab = \Aemp(\Xb) - \Mb \Xb$ is linearly independent almost surely.
\end{lemma}

\begin{proof}
[Proof of \Cref{lemma:total_invertible_meas_zero}]
Since $\alpha N > \dau$, it is sufficient to show that a random matrix with an absolutely continuous distribution is totally invertible \footnote{A matrix is totally invertible if all its square submatrices are invertible.} almost surely.

It is known that for any dimension $m \in \N$, an $m \times m$ square matrix $\Sb$ is singular if $\det(\Sb) = 0$ where entries of $\Sb$ lie within the roots of the polynomial equation specified by the determinant.
Therefore, the set of all singular matrices in $\R^{m \times m}$ has Lebesgue measure zero, 
\begin{align*}
    \lambda\rbr{\csepp{\Sb \in \R^{m \times m}}{\det(\Sb) = 0}} = 0
.\end{align*}
Then, for an absolutely continuous probability measure $\mu$ with respect to $\lambda$, we also have
\[
    \PP_{\mu}\sbr{\Sb \in \R^{m \times m}\ \t{is singular}} = 
    \mu \rbr{\csepp{\Sb \in \R^{m \times m}}{\det(\Sb) = 0}} = 0.
\]
Since a general matrix $\Rb$ contains only finite number of submatrices, when $\Rb$ is drawn from an absolutely continuous distribution, by the union bound, $\PP\sbr{\Rb\ \t{cotains a singular submatrix}} = 0$. 
That is, $\Rb$ is totally invertible almost surely.
\end{proof}

\begin{lemma}
\label{lemma:2layer-relu-input-consistency-exclude-spurious}
Under the assumptions in \Cref{thm:case_2layer_relu_risk}, the hidden layer in the two-layer ReLU network learns $\kernel\rbr{\Deltab}$, the invariant subspace under data augmentations : with high probability,
\begin{align*}
    \rbr{\xb^{\top} \wh{\Bb}^{dac}}_+ = \rbr{\xb^{\top} \projnull \wh{\Bb}^{dac}}_+
    \quad \forall ~ \xb \in \Xcal.
\end{align*}    
\end{lemma} 

\begin{proof}[Proof of \Cref{lemma:2layer-relu-input-consistency-exclude-spurious}]
We will show that for all $\bb_k = \projnull \bb_k + \projrg \bb_k$, $k \in [q]$, $\projrg \bb_k = \b{0}$ with high probability, which then implies that given any $\xb \in \Xcal$, $(\xb^{\top} \bb_k)_+ = (\xb^{\top} \projnull \bb_k)_+$ for all $k \in [q]$.

For any $k \in [q]$ associated with an arbitrary fixed $\bb_k \in \mathbb{S}^{d-1}$, let $\Xb_k \triangleq \Xb_k \projnull + \Xb_k \projrg \in \Xcal^{N_k}$ be the inclusion-wisely maximum row subset of $\Xb$ such that $\Xb_k \bb_k > \b{0}$ element-wisely. 
Meanwhile, we denote $\Aemp(\Xb_k) = \Mb_k \Xb_k \projnull + \Aemp(\Xb_k) \projrg \in \Xcal^{\alpha N_k}$ as the augmentation of $\Xb_k$ where $\Mb_k \in \R^{\alpha N_k \times N_k}$ is the vertical stack of $\alpha$ identity matrices with size $N_k \times N_k$.
Then the DAC constraint implies that $(\Aemp(\Xb_k) - \Mb_k \Xb_k) \projrg \bb_k = \b{0}$.

With \Cref{ass:observable_marginal_distribution}, for a fixed $\bb_k \in \mathbb{S}^{d-1}$, $\PP[\xb^{\top} \bb_k > 0] = \frac{1}{2}$. Then, with the Chernoff bound,
\begin{align*}
    \PP\sbr{N_k < \frac{N}{2} - t} \leq e^{-\frac{2 t^2}{N}},
\end{align*}
which implies that, $N_k \geq \frac{N}{4}$ with high probability.

Leveraging the assumptions in \Cref{thm:case_2layer_relu_risk}, $\alpha N \geq 4 \dau$ implies that $\alpha N_k \geq \dau$. 
Therefore by \Cref{lemma:total_invertible_meas_zero}, $\row\rbr{\Aemp(\Xb_k) - \Mb_k \Xb_k} = \row\rbr{\Deltab}$ with probability $1$.
Thus, $(\Aemp(\Xb_k) - \Mb_k \Xb_k) \projrg \bb_k = \b{0}$ enforces that $\projrg \bb_k = \b{0}$.
\end{proof}

\begin{proof}[Proof of \Cref{thm:case_2layer_relu_risk}]
Conditioned on $\Xb$ and $\Deltab$, we are interested in the excess risk $L(\wh h^{dac}) - L(h^*) = \frac{1}{N} \norm{(\Xb \wh{\Bb}^{dac})_+ \wh{\wb}^{dac} - (\Xb\Bb^*)_+ \wb^*}_2^2$ with randomness on $\epsb$.

We first recall that \Cref{lemma:2layer-relu-input-consistency-exclude-spurious} implies $\wh h^{dac} \in \Hred = \csepp{h(\xb) = \rbr{\xb^\top \Bb}_+ \wb}{\Bb \in \Bcal,~\norm{\wb}_1 \leq C_w}$ where 
\begin{align*}
    \Bcal \triangleq \cbr{\Bb=[\bb_1 \dots \bb_q] ~|~ \norm{\bb_k}=1\ \forall\ k \in [q], (\Xb \Bb)_+ = (\Xb \projnull \Bb)_+ }.
\end{align*}    
Leveraging Equation (21) and (22) in \cite{du2020fewshot},
since $(\Bb^*, \wb^*)$ is feasible under the constraint, by the basic inequality,
\begin{align}
    \label{eq:2layer-relu-basic-ineq}
    \norm{\yb - (\Xb \wh{\Bb}^{dac})_+ \wh{\wb}^{dac}}_2^2 
    \leq
    \norm{\yb - (\Xb \Bb^*)_+ \wb^*}_2^2.
\end{align}
Knowing that $\yb = (\Xb \Bb^*)_+ \wb^* + \epsb$ with $\epsb \sim \Ncal\rbr{\b0, \sigma^2\Ib_N}$, we can rewrite \Cref{eq:2layer-relu-basic-ineq} as
\begin{align*}
    \frac{1}{N} \norm{(\Xb \wh{\Bb}^{dac})_+ \wh{\wb}^{dac} - (\Xb\Bb^*)_+ \wb^*}_2^2
    \le &\frac{2}{N} \epsb^\top \rbr{(\Xb \wh{\Bb}^{dac})_+ \wh{\wb}^{dac} - (\Xb\Bb^*)_+ \wb^*}
    \\
    \le & 4 \sup_{h \in \Hred} \frac{1}{N} \epsb^\top h(\Xb)
\end{align*}    
First, we observe that $\sigma^{-1}\E_{\epsb}\sbr{\sup_{h \in \Hred} \frac{1}{N} \epsb^\top h(\Xb)} = \wh{\fG}_{\Xb}\rbr{\Hred}$ measures the empirical Gaussian width of $\Hred$ over $\Xb$. Moreover, by observing that for any $h \in \Hred$ and $\xb_i \in \Xb$,
\begin{align*}
    &\abbr{h(\xb_i)} 
    \le \nbr{\rbr{\Bb^\top \xb_i}_+}_{\infty} \nbr{\wb}_1 
    \le \max_{k \in [q]} \abbr{\bb_k^\top \projnull \xb_i} \nbr{\wb}_1 
    \le \nbr{\projnull \xb_i}_2 \nbr{\wb}_1,
    \\
    &\frac{1}{N}\nbr{h(\Xb)}^2_2
    = \frac{1}{N} \sum_{i=1}^N \abbr{h(\xb_i)}^2
    \le \nbr{\wb}_1^2 \cdot \frac{1}{N}\sum_{i=1}^N \nbr{\projnull \xb_i}_2^2
    \le C_w^2 \Cnull^2
\end{align*}
and 
\begin{align*}
    &\abbr{\sup_{h \in \Hred} \frac{1}{N} \epsb_1^\top h(\Xb) - \sup_{h \in \Hred} \frac{1}{N} \epsb_2^\top h(\Xb)}
    \\
    \le &\abbr{\sup_{h \in \Hred} \frac{1}{N} h(\Xb)^\top \rbr{\epsb_1 - \epsb_2}}
    \\
    \le &\frac{1}{\sqrt{N}} \nbr{\frac{1}{\sqrt{N}} h(\Xb)}_2 \nbr{\epsb_1 - \epsb_2}_2
    \\
    \le &\frac{C_w \Cnull}{\sqrt{N}} \nbr{\epsb_1 - \epsb_2}_2,
\end{align*}    
we know that the function $\epsb \to \sup_{h \in \Hred} \frac{1}{N} \epsb^\top h(\Xb)$ is $\frac{\Cnull C_w}{\sqrt{N}}$-Lipschitz in $\ell_2$ norm. Therefore, by \cite{wainwright2019} Theorem 2.26, we have that with probability at least $1-\delta$,
\begin{align*}
    \sup_{h \in \Hred} \frac{1}{N} \epsb^\top h(\Xb) \le \sigma \cdot \rbr{\wh{\fG}_{\Xb}\rbr{\Hred} + C_w \Cnull \sqrt{\frac{2 \log(1/\delta)}{N}}}
,\end{align*}
where the empirical Gaussian complexity is upper bounded by
\begin{align*}
    \wh{\fG}_{\Xb}\rbr{\Hred} 
    =
    & \underset{\gb \sim \Ncal(\b{0}, \Ib_N)}{\E} \sbr{
    \underset{\Bb \in \Bcal, \norm{\wb}_1 \leq R}{\sup}\ 
    \frac{1}{N} \gb^{\top} (\Xb\Bb)_+ \wb } \\ 
    \leq
    & \frac{C_w}{N}\ \underset{\gb}{\E} \sbr{
    \underset{\Bb \in \Bcal}{\sup}\ 
    \norm{(\Xb\Bb)_+^{\top} \gb}_{\infty}} \\ 
    =
    & \frac{C_w}{N}\ \underset{\gb}{\E} \sbr{
    \underset{\bb \in \mathbb{S}^{d-1}}{\sup}\ 
    \gb^{\top} \rbr{\Xb \projnull \bb}_+} 
    \quad
    \rbr{\t{\Cref{lemma:tech_gaussian_width_lipschitz}, $(\cdot)_+$ is $1$-Lipschitz} } \\
    \leq 
    & \frac{C_w}{N}\ \underset{\gb}{\E} \sbr{
    \underset{\bb \in \mathbb{S}^{d-1}}{\sup}\ 
    \gb^{\top} \Xb \projnull \bb} \\
    =
    & \frac{C_w}{N}\ \underset{\gb}{\E} 
    \sbr{\norm{\projnull \Xb^{\top} \gb}_2} \\
    \leq
    & \frac{C_w}{N}\ \rbr{\underset{\gb}{\E} \sbr{\norm{\projnull \Xb^{\top} \gb}_2^2}}^{1/2} \\ = 
    & \frac{C_w}{N}\ \sqrt{\tr(\projnull \Xb^{\top} \Xb \projnull)} \\ = 
    & \frac{C_w \Cnull}{\sqrt{N}}.
\end{align*}
\end{proof}

\begin{proof}[Proof of \Cref{coro:two_layer_daug_bound}]
By \Cref{def:daug}, we have with probability at least $1-\delta$ that $\dau = \rank(\projrg) \ge \dau(\delta)$ and $\rank(\projnull) \le d-\dau(\delta)$.
Meanwhile, leveraging \Cref{lemma:sample-population-covariance}, we have that under \Cref{ass:observable_marginal_distribution} and with $N \gg \rho^4 d$, with high probability,
\begin{align*}
    \norm{\frac{1}{N} \projnull \Xb^{\top} \Xb \projnull}_2 
    \leq \norm{\frac{1}{N} \Xb^{\top} \Xb }_2
    \le 1.1 C \lesssim 1.
\end{align*}
Therefore, there exists $\Cnull > 0$ with $\frac{1}{N} \sum_{i=1}^n \norm{\projnull \xb_i}_2^2 \leq \Cnull^2$ such that, with probability at least $1-\delta$,
\begin{align*}
    \Cnull^2 \leq \rbr{d-\dau} \cdot \norm{\frac{1}{N} \projnull \Xb^{\top} \Xb \projnull}_2 \lesssim d-\dau(\delta).
\end{align*}
\end{proof}

\section{Classification with Expansion-based Augmentations}\label{apx:generalized_DAC}

We first recall the multi-class classification problem setup in \Cref{subsec:expansion_based}, while introducing some helpful notions.
For an arbitrary set $\Xcal$, let $\Ycal=[K]$, and $h^*: \Xcal \to [K]$ be the ground truth classifier that partitions $\Xcal$: for each $k \in [K]$, let $\Xcal_k \triangleq \cbr{\xb \in \Xcal ~|~ h^*(\xb)=k}$, with $\Xcal_i \cap \Xcal_j = \emptyset, \forall i \neq j$. 
In addition, for an arbitrary classifier $h: \Xcal \to [K]$, we denote the majority label with respect to $h$ for each class,
\begin{align*}
    \wh{y}_k \triangleq \underset{y \in [K]}{\argmax}\ \PP_{\Pgt} \sbr{h(\xb)=y ~\big|~ \xb \in \Xcal_k} 
    \quad \forall\ k \in [K],
\end{align*}
along with the respective class-wise local and global minority sets,
\begin{align*}
    M_k \triangleq \cbr{\xb \in \Xcal_k ~\big|~ h(\xb) \neq \wh{y}_k} \subsetneq \Xcal_k
    \quad \forall\ k \in [K], \quad
    M \triangleq \bigcup_{k=1}^K M_k.
\end{align*}

Given the marginal distribution $\Pgt\rbr{\xb}$, we introduce the \textit{expansion-based data augmentations} that concretizes \Cref{def:causal_invar_data_aug} in the classification setting:

\begin{definition}[Expansion-based data augmentations, \cite{cai2021theory}]
\label{def:generalized-causal-invariant-data-augmentation}
We call $\Acal:\Xcal \to 2^{\Xcal}$ an augmentation function that induces expansion-based data augmentations if $\Acal$ is class invariant: $\cbr{\xb} \subsetneq \Acal(\xb) \subseteq \cbr{\xb' \in \Xcal ~|~ h^*(\xb) = h^*(\xb')}$ for all $\xb \in \Xcal$.
Let 
\begin{align*}
    \nbh(\xb) \triangleq \cbr{\xb' \in \Xcal ~\big|~ \Acal(\xb) \cap \Acal(\xb') \neq \emptyset},
    \quad
    \nbh(S) \triangleq \cup_{\xb \in S} \nbh(\xb)
\end{align*}
be the neighborhoods of $\xb \in \Xcal$ and $S \subseteq \Xcal$ with respect to $\Acal$.
Then, $\Acal$ satisfies
\begin{enumerate}[nosep,leftmargin=*,label=(\alph*)]
    \item \underline{$(q,\xi)$-constant expansion} if given any $S \subseteq \Xcal$ with $\Pgt\rbr{S} \geq q$ and $\Pgt\rbr{S \cap \Xcal_k} \leq \frac{1}{2}$ for all $k \in [K]$,
    $\Pgt\rbr{\nbh\rbr{S}} \geq \min\cbr{\Pgt\rbr{S},\xi} + \Pgt\rbr{S}$;
    \item \underline{$(a,c)$-multiplicative expansion} if for all $k \in [K]$, given any $S \subseteq \Xcal$ with $\Pgt\rbr{S \cap \Xcal_k} \leq a$,
    $\Pgt\rbr{\nbh\rbr{S} \cap \Xcal_k} \geq \min\cbr{c \cdot \Pgt\rbr{S \cap \Xcal_k},1}$.
\end{enumerate}
\end{definition}

On \Cref{def:generalized-causal-invariant-data-augmentation}, we first point out that the ground truth classifier is invariant throughout the neighborhood: given any $\xb \in \Xcal$, $h^*\rbr{\xb} = h^*\rbr{\xb'}$ for all $\xb' \in \nbh(\xb)$. 
Second, in contrast to the linear regression and two-layer neural network cases where we assume $\Xcal \subseteq R^d$, with the expansion-based data augmentation over a general $\Xcal$, the notion of $\dau$ in \Cref{def:daug} is not well-established. Alternatively, we leverage the concept of constant / multiplicative expansion from \cite{cai2021theory}, and quantify the augmentation strength with parameters $(q,\xi)$ or $(a,c)$. Intuitively, the strength of expansion-based data augmentations is characterized by expansion capability of $\Acal$: for a neighborhood $S \subseteq \Xcal$ of proper size (characterized by $q$ or $a$ under measure $\Pgt$), the stronger augmentation $\Acal$ leads to more expansion in $\nbh(S)$, and therefore larger $\xi$ or $c$. For example in \Cref{def:generalized-causal-invariant-data-augmentation_informal}, we use an expansion-based augmentation function $\Acal$ that satisfies $\rbr{\frac{1}{2}, c}$-multiplicative expansion.

Adapting the existing setting in \cite{wei2021theoretical, cai2021theory}, we concretize the classifier class $\Hcal$ with a function class $\Fcal \subseteq \cbr{f:\Xcal \to \R^K}$ of fully connected neural networks such that $\Hcal = \csepp{h(\xb) \triangleq \argmax_{k \in [K]}\ f(\xb)_k}{f \in \Fcal}$.
To constrain the feasible hypothesis class through the DAC regularization with finite unlabeled samples, we recall the notion of all-layer-margin, $m: \Fcal \times \Xcal \times \Ycal \to \R_{\geq 0}$ (from \cite{wei2021theoretical}) that measures the maximum possible perturbation in all layers of $f$ while maintaining the prediction $y$. 
Precisely, given any $f \in \Fcal$ such that $f\rbr{\xb} = \Wb_p \varphi\rbr{\dots\varphi\rbr{\Wb_1 \xb}\dots}$ for some activation function $\varphi: \R \to \R$ and parameters $\cbr{\Wb_{\iota} \in \R^{d_{\iota} \times d_{\iota-1}}}_{\iota=1}^p$, we can write $f = f_{2p-1} \circ \dots \circ f_1$ where $f_{2\iota-1}(\xb) = \Wb_{\iota} \xb$ for all $\iota \in [p]$ and $f_{2\iota}(\zb)=\varphi(\zb)$ for $\iota \in [p-1]$.
For an arbitrary set of perturbation vectors $\deltab = \rbr{\deltab_1,\dots,\deltab_{2p-1}}$ such that $\deltab_{2\iota-1}, \deltab_{2\iota} \in \R^{d_{\iota}}$ for all $\iota$, let $f(\xb,\deltab)$ be the perturbed neural network defined recursively such that
\begin{align*}
    & \wt{\zb}_1 = f_1\rbr{\xb} + \norm{\xb}_2 \deltab_1, \\
    & \wt{\zb}_{\iota} = f_{\iota}\rbr{\wt{\zb}_{\iota-1}} + \norm{\wt{\zb}_{\iota-1}}_2 \deltab_{\iota}
    \quad \forall\ \iota=2,\dots,2p-1, \\
    & f(\xb,\deltab) = \wt{\zb}_{2p-1}.
\end{align*}   
The all-layer-margin $m(f,\xb,y)$ measures the minimum norm of the perturbation $\deltab$ such that $f(\xb,\deltab)$ fails to provide the classification $y$,
\begin{align}
    \label{eq:def-all-layer-margin}
    m(f,\xb,y) \triangleq
    \underset{\deltab = \rbr{\deltab_1,\dots,\deltab_{2p-1}}}{\min}
    \sqrt{\sum_{\iota=1}^{2p-1} \norm{\deltab_{\iota}}_2^2}
    \quad \t{s.t.} \quad
    \underset{k \in [K]}{\argmax}\ f(\xb,\deltab)_k \neq y.
\end{align}
With the notion of all-layer-margin established, for any $\Acal:\Xcal \to 2^{\Xcal}$ that satisfies conditions in \Cref{def:generalized-causal-invariant-data-augmentation}, the robust margin is defined as
\begin{align*}
    m_{\Acal}(f,\xb) \triangleq \underset{\xb' \in \Acal(\xb)}{\sup}\  
    m\rbr{f, \xb', \argmax_{k \in [K]}\ f(\xb)_k}.
\end{align*}
Intuitively, the robust margin measures the maximum possible perturbation in all-layer weights of $f$ such that predictions on all data augmentations of $\xb$ remain consistent. For instance, $m_{\Acal}(f,\xb) > 0$ is equivalent to enforcing $h(\xb) = h(\xb')$ for all $\xb' \in \Acal\rbr{\xb}$.

To achieve finite sample guarantees, DAC regularization requires stronger consistency conditions than merely consistent classification outputs ($\ie$, $m_{\Acal}(f,\xb) > 0$). Instead, we enforce $m_{\Acal}(f,\xb) > \tau$ for any $0 < \tau < \max_{f \in \Fcal}\ \inf_{\xb \in \Xcal} m_{\Acal}(f, \xb)$\footnote{The upper bound on $\tau$ ensures the proper learning setting, $\ie$, there exists $f \in \Fcal$ such that $m_\Acal\rbr{f,\xb} > \tau$ for all $\xb \in \Xcal$.} over an finite set of unlabeled samples $\Xb^u$ that is independent of $\Xb$ and drawn $\iid$ from $P(\xb)$.
Then, learning the classifier with zero-one loss $l_{01}\rbr{h(\xb),y} = \b1\cbr{h(\xb) \neq y}$ from a class of $p$-layer fully connected neural networks with maximum width $q$,
\begin{align*}
    \Fcal = 
    \csepp{f: \Xcal \to \R^K}
    {f = f_{2p-1} \circ \dots \circ f_1,\ f_{2\iota-1}(\xb) = \Wb_{\iota} \xb,\ f_{2\iota}(\zb)=\varphi(\zb)},
\end{align*}
where $\Wb_{\iota} \in \R^{d_{\iota} \times d_{\iota-1}}$ for all $\iota \in [p]$, and $q \triangleq \max_{\iota=0,\dots,p} d_{\iota}$, we solve
\begin{align}\label{eq:gdac_finite}
    \hgdacfin\ \triangleq\ 
    & \underset{h \in \Hcal}{\argmin}\ 
    \wh{L}^{dac}_{01}(h) = 
    \frac{1}{N} \sum_{i=1}^N \b{1}\cbr{h\rbr{\xb_i} \neq h^*\rbr{\xb_i}} 
    \\
    & \t{s.t.} \quad
    m_{\Acal}(f,\xb^u_i)> \tau \quad \forall\ i \in [\abbr{\Xb^u}] \nonumber 
\end{align}
for any $0 < \tau < \max_{f \in \Fcal}\ \inf_{\xb \in \Xcal} m_{\Acal}(f, \xb)$.
The corresponding reduced function class is given by
\begin{align*}
    \Hred \triangleq 
    \csepp{h \in \Hcal}
    {m_{\Acal}(f,\xb^u_i)> \tau \quad \forall\ i \in [\abbr{\Xb^u}]}.
\end{align*}
Specifically, with $\mu \triangleq \sup_{h \in \Hred} \PP_{\Pgt} \sbr{\exists\ \xb' \in \Acal(\xb): h(\xb) \neq h(\xb')}$, \cite{wei2021theoretical, cai2021theory} demonstrate the following for $\Hred$:
\begin{proposition}
[\cite{wei2021theoretical} Theorem 3.7, \cite{cai2021theory} Proposition 2.2]
\label{prop:non-robust-upper-bound}
For any $\delta \in (0,1)$, with probability at least $1-\delta/2$ (over $\Xb^u$), 
\begin{align*}
    \mu 
    \leq \wt{O} \rbr{\frac{\sum_{\iota=1}^p \sqrt{q} \norm{\Wb_{\iota}}_F}{\tau \sqrt{\abbr{\Xb^u}}} + \sqrt{\frac{\log\rbr{1/\delta} + p \log \abbr{\Xb^u}}{\abbr{\Xb^u}}}}
,\end{align*}
where $\wt{O}\rbr{\cdot}$ hides polylogarithmic factors in $\abbr{\Xb^u}$ and $d$.
\end{proposition}   

Leveraging the existing theory above on finite sample guarantee of the maximum possible inconsistency, we have the following.
\begin{theorem}[Formal restatement of \Cref{thm:generalized-dac-finite-unlabeled_informal} on classification with DAC]
\label{thm:generalized-dac-finite-unlabeled}
Learning the classifier with DAC regularization in \Cref{eq:gdac_finite} provides that, for any $\delta \in (0,1)$, with probability at least $1-\delta$,
\begin{align}
    \label{eq:generalization-bound-generalized-dac-finite}
    L_{01}\rbr{\hgdacfin} - L_{01}\rbr{h^*} \leq 
    4 \Rndac + \sqrt{\frac{2 \log(4/\delta)}{N}},
\end{align}
where with $0 < \mu < 1$ defined in \Cref{prop:non-robust-upper-bound}, for any $0 \leq q < \frac{1}{2}$ and $c > 1+4\mu$, 
\begin{enumerate}[label=(\alph*),nosep]
    \item when $\Acal$ satisfies $(q,2\mu)$-constant expansion, $\Rndac \leq \sqrt{\frac{2K \log K}{N} + 2K \max\cbr{q,2\mu}}$;
    \item when $\Acal$ satisfies $(\frac{1}{2},c)$-multiplicative expansion, $\Rndac \leq \sqrt{\frac{2K \log K}{N} + \frac{4 K \mu}{\min\cbr{c-1,1}}}$.
\end{enumerate}
\end{theorem} 

First, to quantify the function class complexity and relate it to the generalization error, we leverage the notion of Rademacher complexity and the associated standard generalization bound.
\begin{lemma}\label{lemma:rademacher_generalization_bound}
Given a fixed function class $\Hred$ (\ie, conditioned on $\Xb^u$) and a $B$-bounded and $C_l$-Lipschitz loss function $l$, let $\wh L(h) = \frac{1}{N} \sum_{i=1}^N l(h(\xb_i),y_i)$, $L(h) = \E\sbr{l(h(\xb_i),y_i)}$, and $\wh h^{dac} = \argmin_{h \in \Hred} \wh L(h)$. Then for any $\delta \in (0,1)$, with probability at least $1-\delta$ over $\Xb$, 
\begin{align*}
    L(\widehat h^{dac}) - L(h^*) \le & 4 C_l \cdot \fR_N\rbr{\Hred} + \sqrt{\frac{2B^2\log(4/\delta)}{N}}.
\end{align*}
\end{lemma}

\begin{proof}[Proof of \Cref{lemma:rademacher_generalization_bound}]
We first decompose the expected excess risk as
\[
L(\wh{h}^{dac}) - L(h^*) = 
\rbr{L(\wh{h}^{dac}) - \wh{L}(\wh{h}^{dac})} + \rbr{\wh{L}(\wh{h}^{dac}) - \wh{L}(h^*)} + \rbr{\wh{L}(h^*) - L(h^*)},
\]
where $\wh{L}(\wh{h}^{dac}) - \wh{L}(h^*) \leq 0$ by the basic inequality. Since both $\wh h^{dac}, h^* \in \Hred$, we then have
\begin{align*}
    L(\wh{h}^{dac}) - L(h^*) 
    \leq 2 \sup_{h \in \Hred}\ \abbr{L(h) - \wh{L}(h)}.
\end{align*}    
Let $g^+(\Xb,\yb) = \sup_{h \in \Hred}: L(h) - \wh{L}(h)$ and 
$g^-(\Xb,\yb) = \sup_{h \in \Hred}: - L(h) + \wh{L}(h)$. Then,
\[
\PP\sbr{L(\wh{h}^{dac}) - L(h^*) \geq \epsilon} \leq
\PP\sbr{g^+(\Xb,\yb) \geq \frac{\eps}{2}} + 
\PP\sbr{g^-(\Xb,\yb) \geq \frac{\eps}{2} }.
\]
We will derive a tail bound for $g^+(\Xb,\yb)$ with the standard inequalities and symmetrization argument \cite{wainwright2019,bartlett2003}, while the analogous statement holds for $g^-(\Xb,\yb)$.

Let $(\Xb^{(1)}, \yb^{(1)})$ be a sample set generated by replacing an arbitrary sample in $(\Xb, \yb)$ with an independent sample $(\xb, y) \sim P(\xb,y)$. Since $l$ is $B$-bounded, we have $\abbr{g^+(\Xb, \yb) - g^+(\Xb^{(1)}, \yb^{(1)})} \leq B/N$. Then, via McDiarmid's inequality \cite{bartlett2003},
\[
    \PP\sbr{g^+(\Xb, \yb) \geq \E[g^+(\Xb, \yb)] + t} 
    \leq \exp\rbr{-\frac{2 N t^2}{B^2}}.
\]
For an arbitrary sample set $\rbr{\Xb,\yb}$, let $\wh{L}_{\rbr{\Xb,\yb}}\rbr{h} = \frac{1}{N} \sum_{i=1}^N l\rbr{h(\xb_i),y_i}$ be the empirical risk of $h$ with respect to $\rbr{\Xb,\yb}$.
Then, by a classical symmetrization argument (e.g., proof of \cite{wainwright2019} Theorem 4.10), we can bound the expectation: for an independent sample set $\rbr{\Xb',\yb'} \in \Xcal^N \times \Ycal^N$ drawn $\iid$ from $\Pgt$, 
\begin{align*}
    \E\sbr{g^+(\Xb, \yb)} 
    = & \E_{(\Xb,\yb)} \sbr{ \sup_{h \in \Hred}\ L(h) - \wh{L}_{(\Xb,\yb)}(h)}
    \\
    = & \E_{(\Xb,\yb)} \sbr{ \sup_{h \in \Hred}\ \E_{(\Xb',\yb')} \sbr{\wh L_{(\Xb',\yb')}(h)} - \wh{L}_{(\Xb,\yb)}(h)}
    \\
    = & \E_{(\Xb,\yb)} \sbr{ \sup_{h \in \Hred}\ \E_{(\Xb',\yb')} \sbr{\wh L_{(\Xb',\yb')}(h) - \wh{L}_{(\Xb,\yb)}(h) ~\middle|~ \rbr{\Xb,\yb}} }
    \\
    \leq & \E_{(\Xb,\yb)} \sbr{ \E_{(\Xb',\yb')} \sbr{ \sup_{h \in \Hred}\ \wh L_{(\Xb',\yb')}(h) - \wh{L}_{(\Xb,\yb)}(h) ~\middle|~ \rbr{\Xb,\yb}} } 
    \\
    & \rbr{\t{Law of iterated conditional expectation}}
    \\
    = & \E_{\rbr{\Xb,\yb,\Xb',\yb'}} \sbr{\sup_{h \in \Hred}\ \wh L_{(\Xb',\yb')}(h) - \wh{L}_{(\Xb,\yb)}(h) }
\end{align*}
Since $\rbr{\Xb,\yb}, \rbr{\Xb',\yb'}$ are drawn $\iid$ from $\Pgt$, we can introduce $\iid$ Rademacher random variables $\rb = \cbr{r_i \in \cbr{-1,+1} ~|~ i \in [N]}$ (independent of both $\rbr{\Xb,\yb}$ and $\rbr{\Xb',\yb'}$) such that
\begin{align*}
    \E\sbr{g^+(\Xb, \yb) } 
    \leq & \E_{\rbr{\Xb,\yb,\Xb',\yb',\rb}} \sbr{\sup_{h \in \Hred}\ \frac{1}{N} \sum_{i=1}^N r_i \cdot \rbr{l\rbr{h\rbr{\xb'_i},y'_i} - l\rbr{h\rbr{\xb_i},y_i}} }
    \\
    \leq & 2\ E_{\rbr{\Xb,\yb,\rb}} \sbr{ \sup_{h \in \Hred}\ \frac{1}{N} \sum_{i=1}^N r_i \cdot l\rbr{h\rbr{\xb_i},y_i} }
    \\
    \leq & 2\ \fR_N\rbr{l \circ \Hred}
\end{align*}
where $l \circ \Hred = \cbr{l(h(\cdot), \cdot): \Xcal \times \Ycal \to \R: h \in \Hred}$ is the loss function class, and 
\begin{align*}
    \fR_N\rbr{\Fcal} 
    \triangleq E_{\rbr{\Xb,\yb,\rb}} \sbr{ \sup_{f \in \Fcal}\ \frac{1}{N} \sum_{i=1}^N r_i \cdot f\rbr{\xb_i,y_i}}
\end{align*}
denotes the Rademacher complexity. 
Analogously, $\E[g^-(\Xb, \yb)] \leq 2\fR_N\rbr{l \circ \Hred}$. 
Therefore, assuming that $\dacop(\Hcal) \subseteq \Hred(\Hcal)$ holds, with probability at least $1-\delta/2$,
\[
L(\wh{h}^{dac}) - L(h^*) \leq 4 \fR_N\rbr{l \circ \Hred} + \sqrt{\frac{2 B^2 \log(4/\delta)}{N}}
\]

Finally, since $l(\cdot, y)$ is $C_l$-Lipschitz for all $y \in \Ycal$, by \cite{ledoux2013} Theorem 4.12, we have $\fR_N\rbr{l \circ \Hred} \leq C_l \cdot \fR_N\rbr{\Hred} $.
\end{proof}

\begin{lemma}
[\cite{cai2021theory}, Lemma A.1]
\label{lemma:minority-set-upper-bound-expansion-assumption}
For any $h \in \Hred$, when $\Pgt$ satisfies
\begin{enumerate}[label=(\alph*), nosep]
    \item $\rbr{q, 2\mu}$-constant expansion with $q < \frac{1}{2}$, $\Pgt\rbr{M} \leq \max\cbr{q, 2\mu}$;
    \item $\rbr{\frac{1}{2},c}$-multiplicative expansion with $c>1+4\mu$, $\Pgt\rbr{M} \leq \max\cbr{\frac{2\mu}{c-1},2\mu}$.
\end{enumerate}
\end{lemma} 

\begin{proof}
[Proof of \Cref{lemma:minority-set-upper-bound-expansion-assumption}]
We start with the proof for \Cref{lemma:minority-set-upper-bound-expansion-assumption} (a).
By definition of $M_k$ and $\wh{y}_k$, we know that $M_k = M \cap \Xcal_k \leq \frac{1}{2}$. Therefore, for any $0 < q < \frac{1}{2}$, one of the following two cases holds:
\begin{enumerate}[label=(\roman*), nosep]
    \item $\Pgt\rbr{M} < q$;
    \item $\Pgt\rbr{M} \geq q$.
    Since $\Pgt\rbr{M \cap \Xcal_k} < \frac{1}{2}$ for all $k \in [K]$ holds by construction, with the $\rbr{q, 2\mu}$-constant expansion, $\Pgt\rbr{\nbh\rbr{M}} \geq \min\cbr{\Pgt\rbr{M}, 2\mu} + \Pgt\rbr{M}$.
    
    Meanwhile, since the ground truth classifier $h^*$ is invariant throughout the neighborhoods, $\nbh\rbr{M_k} \cap \nbh\rbr{M_{k'}} = \emptyset$ for $k \neq k'$, and therefore $\nbh\rbr{M} \backslash M = \bigcup_{k=1}^K \nbh\rbr{M_k} \backslash M_k$ with each $\nbh\rbr{M_k} \backslash M_k$ disjoint.
    Then, we observe that for each $\xb \in \nbh\rbr{M} \backslash M$, here exists some $k = h^*\rbr{\xb}$ such that $\xb \in \nbh\rbr{M_k} \backslash M_k$.
    $\xb \in \Xcal_k \backslash M_k$ implies that $h\rbr{\xb} = \wh{y}_k$, while
    $\xb \in \nbh\rbr{M_k}$ suggests that there exists some $\xb' \in \Acal\rbr{\xb} \cap \Acal\rbr{\xb''}$ where $\xb'' \in M_k$ such that either 
    $h\rbr{\xb'} = \wh{y}_k$ and $h\rbr{\xb'} \neq h\rbr{\xb''}$ for $\xb' \in \Acal\rbr{\xb''}$, 
    or 
    $h\rbr{\xb'} \neq \wh{y}_k$ and $h\rbr{\xb'} \neq h\rbr{\xb}$ for $\xb' \in \Acal\rbr{\xb}$. Therefore, we have 
    \[
    \Pgt\rbr{\nbh\rbr{M} \backslash M} \leq 
    2 \PP_{\Pgt}\sbr{\exists\ \xb' \in \Acal(\xb)\ \t{s.t.}\ h(\xb) \neq h(\xb')} \leq 2\mu.
    \]
    Moreover, since $\Pgt\rbr{\nbh\rbr{M}} - \Pgt\rbr{M} \leq \Pgt\rbr{\nbh\rbr{M} \backslash M} \leq 2 \mu$, we know that 
    \[
    \min\cbr{\Pgt\rbr{M}, 2\mu} + \Pgt\rbr{M} \leq 
    \Pgt\rbr{\nbh\rbr{M}} \leq 
    \Pgt\rbr{M} + 2 \mu. 
    \]
    That is, $\Pgt\rbr{M} \leq 2 \mu$.
\end{enumerate}
Overall, we have $\Pgt\rbr{M} \leq \max\cbr{q, 2\mu}$.

To show \Cref{lemma:minority-set-upper-bound-expansion-assumption} (b), we recall from \cite{wei2021theoretical} Lemma B.6 that for any $c>1+4\mu$, $\rbr{\frac{1}{2},c}$-multiplicative expansion implies $\rbr{\frac{2\mu}{c-1}, 2\mu}$-constant expansion. 
Then leveraging the proof for \Cref{lemma:minority-set-upper-bound-expansion-assumption} (a), with $q = \frac{2\mu}{c-1}$, we have $\Pgt\rbr{M} \leq \max\cbr{\frac{2\mu}{c-1}, 2\mu}$.
\end{proof}

\begin{proof}
[Proof of \Cref{thm:generalized-dac-finite-unlabeled}]
To show \Cref{eq:generalization-bound-generalized-dac-finite}, we leverage \Cref{lemma:rademacher_generalization_bound} and observe that $B = 1$ with the zero-one loss. Therefore, conditioned on $\Hred$ (which depends only on $\Xb^u$ but not on $\Xb$), for any $\delta \in (0,1)$, with probability at least $1-\delta/2$,
\begin{align*}
    L_{01}\rbr{\hgdacfin} - L_{01}\rbr{h^*} \leq 
    4 \fR_N\rbr{l_{01} \circ \Hred} + \sqrt{\frac{2 \log(4/\delta)}{N}}
.\end{align*}
For the upper bounds of the Rademacher complexity, let $\wt{\mu} \triangleq \sup_{h \in \Hred} \Pgt\rbr{M}$ where $M$ denotes the global minority set with respect to $h \in \Hred$.
\Cref{lemma:minority-set-upper-bound-expansion-assumption} suggests that
\begin{enumerate}[label=(\alph*), nosep]
    \item when $\Pgt$ satisfies $(q,2\mu)$-constant expansion for some $q < \frac{1}{2}$, $\wt{\mu} \leq \max\cbr{q,2\mu}$; while
    \item when $\Pgt$ satisfies $(\frac{1}{2},c)$-multiplicative expansion for some $c > 1+4\mu$,
    $\wt{\mu} \leq \frac{2\mu}{\min\cbr{c-1,1}}$.
\end{enumerate}
Then, it is sufficient to show that, conditioned on $\Hred$, 
\begin{align}\label{eq:pf_Rndac}
    \fR_N\rbr{l_{01} \circ \Hred} \leq \sqrt{\frac{2K \log K}{N} + 2K\wt{\mu}}. 
\end{align}    

To show this, we first consider a fixed set of $n$ observations in $\Xcal$, $\Xb = \sbr{\xb_1,\dots,\xb_N}^{\top} \in \Xcal^N$. 
Let the number of distinct behaviors over $\Xb$ in $\Hred$ be
\begin{align*}
    \mathfrak{s}\rbr{\Hred, \Xb} \triangleq
    \abs{\cbr{\sbr{h\rbr{\xb_1},\dots,h\rbr{\xb_N}} ~\big|~ h \in \Hred}}.
\end{align*}
Then, by the Massart's finite lemma, the empirical rademacher complexity with respect to $\Xb$ is upper bounded by
\begin{align*}
    \wh{\fR}_{\Xb}\rbr{l_{01} \circ \Hred} \leq \sqrt{\frac{2 \log \mathfrak{s}\rbr{\Hred, \Xb}}{N}}.
\end{align*}
By the concavity of $\sqrt{\log\rbr{\cdot}}$, we know that, 
\begin{align}
\label{eq:rademacher-upper-from-shatter}
    \fR_N\rbr{l_{01} \circ \Hred} = 
    & \E_{\Xb} \sbr{\wh{\fR}_{\Xb}\rbr{l_{01} \circ \Hred}} \leq 
    \E_{\Xb} \sbr{\sqrt{\frac{2 \log \mathfrak{s}\rbr{\Hred, \Xb}}{N}}} 
    \nonumber \\ \leq & 
    \sqrt{\frac{2 \log \E_{\Xb} \sbr{\mathfrak{s}\rbr{\Hred, \Xb}}}{N}}.
\end{align}    
Since $\Pgt\rbr{M} \leq \wt{\mu} \leq \frac{1}{2}$ for all $h \in \Hred$, we have that, conditioned on $\Hred$,
\begin{align*}
    \E_{\Xb} \sbr{\mathfrak{s}\rbr{\Hred, \Xb}} \leq & 
    \sum_{r=0}^N \binom{N}{r} \wt{\mu}^r \rbr{1-\wt{\mu}}^{N-r} \cdot K^K \cdot K^r
    \\ \leq &
    K^K \sum_{r=0}^N \binom{N}{r} \rbr{\wt{\mu}K}^r \rbr{1-\wt{\mu}}^{N-r}
    \\ = &
    K^K \rbr{1-\wt{\mu} + K\wt{\mu}}^N
    \\ \leq &
    K^K \cdot e^{K N \wt{\mu}}.
\end{align*}    
Plugging this into \Cref{eq:rademacher-upper-from-shatter} yields \Cref{eq:pf_Rndac}. Finally, the randomness in $\Hred$ is quantified by $\wt\mu, \mu$, and upper bounded by \Cref{prop:non-robust-upper-bound}. 
\end{proof}

\section{Supplementary Application: Domain Adaptation}\label{apx:case_ood}

As a supplementary example, we demonstrate the possible failure of DA-ERM, and alternatively how DAC regularization can serve as a remedy. Concretely, we consider an illustrative linear regression problem in the domain adaptation setting: with training samples drawn from a source distribution $P^s$ and generalization (in terms of excess risk) evaluated over a related but different target distribution $P^t$. 
With distinct $\E_{P^s}\sbr{y|\xb}$ and $\E_{P^t}\sbr{y|\xb}$, we assume the existence of an unknown but unique inclusionwisely maximal invariant feature subspace $\Xcal_r \subset \Xcal$ such that $P^s\sbr{y|\xb \in \Xcal_r} = P^t\sbr{y|\xb \in \Xcal_r}$, we aim to demonstrate the advantage of the DAC regularization over the ERM on augmented training set, with a provable separation in the respective excess risks.

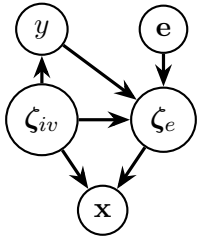
\begin{figure}[!ht]
    \begin{center}
    \begin{tikzpicture}
    \begin{scope}[every node/.style={circle,thick,draw}]
        \node (x) at (0,0) {$\xb$};
        \node (xc) at (-0.8,1.2) {$\zetab_{iv}$};
        \node (xe) at (0.8,1.2) {$\zetab_e$};
        \node (e) at (0.8,2.4) {$\eb$};
        \node (y) at (-0.8,2.4) {$y$};
    \end{scope}
    
    \begin{scope}[>={Stealth[black]},
                  every node/.style={fill=white,circle},
                  every edge/.style={draw=black,very thick}]
        \path [->] (xc) edge (y);
        \path [->] (y) edge (xe);
        \path [->] (xc) edge (xe);
        \path [->] (e) edge (xe);
        \path [->] (xc) edge (x);
        \path [->] (xe) edge (x);
    \end{scope}
    \end{tikzpicture}
    \end{center}
    \caption{Causal graph shared by $P^s$ and $P^t$.}
    \label{fig:data-distribution-causal-graph}
\end{figure}

\paragraph{Source and target distributions.}
Formally, the source and target distributions are concretized with the causal graph in \Cref{fig:data-distribution-causal-graph}. For both $P^s$ and $P^t$, the observable feature $\xb$ is described via a linear generative model in terms of two latent features, the `invariant' feature $\zetab_{iv} \in \R^{d_{iv}}$ and the `environmental' feature $\zetab_e \in \R^{d_e}$:
\begin{align*}
    \xb = g(\zetab_{iv}, \zetab_e) \triangleq \Sb \bmat{\zetab_{iv}; \zetab_e} = \Sb_{iv} \zetab_{iv}  + \Sb_e \zetab_e,
\end{align*}
where $\Sb = \bmat{\Sb_{iv},\Sb_e} \in \R^{d \times (d_{iv} + d_e)}$ ($d_{iv}+d_e \leq d$) consists of orthonormal columns.
Let the label $y$ depends only on the invariant feature $\zetab_{iv}$ for both domains,
\begin{align*}
    y = \rbr{\thetab^{*}}^{\top} \xb + z 
    = \rbr{\thetab^{*}}^{\top} \Sb_{iv} \zetab_{iv} + z,
    \quad 
    z \sim \Ncal\rbr{0, \sigma^2}, 
    \quad
    z \perp \zetab_{iv},
\end{align*}
for some $\thetab^* \in \range\rbr{\Sb_{iv}}$ such that $P^s\sbr{y|\zetab_{iv}} = P^t\sbr{y|\zetab_{iv}}$, while the environmental feature $\zetab_e$ is conditioned on $y$, $\zetab_{iv}$, (along with the Gaussian noise $z$), and varies across different domains $\eb$ with $\E_{P^s}\sbr{y|\xb} \neq \E_{P^t}\sbr{y|\xb}$. In other words, with the square loss $l(h(\xb), y) = \frac{1}{2}(h(\xb)-y)^2$, the optimal hypotheses that minimize the expected excess risk over the source and target distributions are distinct. Therefore, learning via the ERM with training samples from $P^s$ can overfit the source distribution, in which scenario identifying the invariant feature subspace $\range\rbr{\Sb_{iv}}$ becomes indispensable for achieving good generalization in the target domain.

For $P^s$ and $P^t$, we assume the following regularity conditions: 
\begin{assumption}[Regularity conditions for $P^s$ and $P^t$]
\label{ass:case_ood_distribution}
Let $P^s$ satisfy \Cref{ass:observable_marginal_distribution}. While $P^t$ satisfies that $\E_{P^t}[\xb\xb^{\top}] \succ 0$, and
\begin{enumerate}[label=(\alph*),nosep,leftmargin=*]
    \item for the invariant feature, $c_{t,iv} \Ib_{d_{iv}} \aleq \E_{P^t}[\zetab_{iv} \zetab_{iv}^{\top}] \aleq C_{t,iv} \Ib_{d_{iv}}$ for some $C_{t,iv} \geq c_{t,iv} = \Theta(1)$; 
    \item for the environmental feature, $\E_{P^t}[\zetab_{e} \zetab_{e}^{\top}] \ageq c_{t,e} \Ib_{d_e}$ for some $c_{t,e} > 0$, and $\E_{P^t}\sbr{z \cdot \zetab_e} = \b{0}$.
\end{enumerate}
\end{assumption}

\paragraph{Training samples and data augmentations.}
Let $\Xb=\sbr{\xb_1;\dots;\xb_N}$ be a set of $N$ samples drawn $\iid$ from $P^s(\xb)$ such that $\yb=\Xb\thetab^* + \zb$ where $\zb \sim \Ncal(\b0,\sigma^2\Ib_N)$.
Recall that we denote the augmented training sets, including/excluding the original samples, respectively, with
\begin{align*}
    &\wt\Acal(\Xb) = \sbr{\xb_{1}; \cdots; \xb_{N}; \xb_{1,1}; \cdots; \xb_{N,1}; \cdots; \xb_{1, \alpha}; \cdots; \xb_{N, \alpha}} \in \Xcal^{(1+\alpha) N}, 
    \\
    &\Aemp(\Xb) = \sbr{\xb_{1,1}; \cdots; \xb_{N,1}; \cdots; \xb_{1, \alpha}; \cdots; \xb_{N, \alpha}} \in \Xcal^{\alpha N}.
\end{align*}
In particular, we consider a set of augmentations that only perturb the environmental feature $\zetab_e$, while keep the invariant feature $\zetab_{iv}$ intact:
\begin{align}\label{eq:def_envi_data_aug}
    \Sb_{iv}^{\top} \xb_i = \Sb_{iv}^{\top} \xb_{i,j},
    \quad
    \Sb_{e}^{\top} \xb_i \neq \Sb_{e}^{\top} \xb_{i,j}
    \quad
    \forall\ i \in [n],\ j \in [\alpha].
\end{align}
We recall the notion $\Deltab\triangleq\Aemp\rbr{\Xb}-\Mb\Xb$ such that $\dau \triangleq \rank\rbr{\Deltab} = \rank\rbr{\wt\Acal\rbr{\Xb}-\wt\Mb\Xb}$ ($0 \leq \dau \leq d_e$), and assume that $\Xb$ and $\Aemp(\Xb)$ are representative enough:
\begin{assumption}[Diversity of $\Xb$ and $\Aemp(\Xb)$]
\label{ass:case_ood_sample}
$(\Xb,\yb) \in \Xcal^n \times \Ycal^n$ is sufficiently large with $n \gg \rho^4 d$, $\thetab^* \in \row(\Xb)$, and $\dau = d_e$.
\end{assumption}

\paragraph{Excess risks in target domain.}
Learning from the linear hypothesis class $\Hcal = \csepp{h(\xb) = \xb^{\top}\thetab}{\thetab \in \R^{d}}$, with the DAC regularization on $h\rbr{\xb_i}=h\rbr{\xb_{i,j}}$, we have
\begin{align*}
    &\wh{\thetab}^{dac}\ =\ 
    \underset{\thetab \in \Hred}{\argmin}\
    \frac{1}{2N} \norm{\yb - \Xb\thetab}_2^2,
    \quad 
    \Hred = 
    \csepp{h\rbr{\xb}= \thetab^{\top} \xb} 
    {\Deltab\thetab = \b0},
\end{align*}
while with the ERM on augmented training set,
\begin{align*}
    \wh{\thetab}^{\herm}\ =\ 
    & \underset{\thetab \in \R^{d}}{\argmin}\ 
    \frac{1}{2 (1+\alpha) N} \norm{\wt\Mb\yb - \wt\Acal(\Xb)\thetab}_2^2,
\end{align*}    
where $\Mb$ and $\wt\Mb$ denote the vertical stacks of $\alpha$ and $1+\alpha$ identity matrices of size $n \times n$, respectively as denoted earlier.

We are interested in the excess risk on $P^t$: $L_t\rbr{\thetab} - L_t\rbr{\thetab^*}$ where $L_t\rbr{\thetab} \triangleq \E_{P^t\rbr{\xb,y}} \sbr{\frac{1}{2} (y - \xb^{\top} \thetab)^2}$.
\begin{theorem}[Domain adaptation with DAC]
\label{thm:domain-adaption-data-aug-consistency}
Under \Cref{ass:case_ood_distribution}(a) and \Cref{ass:case_ood_sample}, $\wh{\thetab}^{dac}$ satisfies that, with constant probability,
\begin{align}
    \label{eq:domain-adaption-data-aug-consistency}
    \E_{P^s}\sbr{L_t(\wh{\thetab}^{dac}) - L_t(\thetab^*)} 
    \ \lesssim\ 
    \frac{\sigma^2 d_{iv}}{N}.
\end{align}
\end{theorem}

\begin{theorem}[Domain adaptation with ERM on augmented samples]
\label{thm:domain-adaption-plain-data-aug}
Under \Cref{ass:case_ood_distribution} and \Cref{ass:case_ood_sample}, $\wh{\thetab}^{dac}$ and $\wh{\thetab}^{\herm}$ satisfies that, 
\begin{align}
    \label{eq:domain-adaption-plain-data-aug-excess-risk}
    \E_{P^s}\sbr{L_t(\wh{\thetab}^{\herm}) - L_t(\thetab^*)} 
    \ \geq \ 
    \E_{P^s}\sbr{L_t(\wh{\thetab}^{dac}) - L_t(\thetab^*)}
    + c_{t,e} \cdot \EER_e,
\end{align}
for some $\EER_e > 0$.
\end{theorem}   
In contrast to $\wh{\thetab}^{dac}$ where the DAC constraints enforce $\Sb_e^{\top} \wh{\thetab}^{dac} = \b{0}$ with a sufficiently diverse $\Aemp\rbr{\Xb}$ (\Cref{ass:case_ood_sample}), the ERM on augmented training set fails to filter out the environmental feature in $\wh{\thetab}^{\herm}$: $\Sb_e^{\top} \wh{\thetab}^{\herm} \neq \b{0}$. As a consequence, the expected excess risk of $\wh{\thetab}^{\herm}$ in the target domain can be catastrophic when $c_{t,e} \to \infty$, as instantiated by \Cref{example:ood}.

\paragraph{Proofs and instantiation.}
Recall that for $\Deltab \triangleq \Aemp(\Xb) - \Mb\Xb$, $\projnull \triangleq \Ib_d - \Deltab^{\dagger} \Deltab$ denotes the orthogonal projector onto the dimension-$(d-\dau)$ null space of $\Deltab$. 
Furthermore, let $\Pb_{iv} \triangleq \Sb_{iv} \Sb_{iv}^{\top}$ and $\Pb_{e} \triangleq \Sb_{e} \Sb_{e}^{\top}$ be the orthogonal projectors onto the invariant and environmental feature subspaces, respectively, such that $\xb = \Sb_{iv} \zetab_{iv}  + \Sb_e \zetab_e = \rbr{\Pb_{iv}+\Pb_e} \xb$ for all $\xb$.

\begin{proof}[Proof of \Cref{thm:domain-adaption-data-aug-consistency}]

By construction \Cref{eq:def_envi_data_aug}, $\Deltab\Pb_{iv} = \b{0}$, and it follows that $\Pb_{iv} \aleq \projnull$. 
Meanwhile from \Cref{ass:case_ood_sample}, $\dau=d_e$ implies that $\dim\rbr{\projnull} = d_{iv}$. 
Therefore, $\Pb_{iv} = \projnull$, and the data augmentation consistency constraints can be restated as
\begin{align*}
\Hred = \csepp{h\rbr{\xb} = \thetab^{\top} \xb}{\projnull\thetab = \thetab}
= 
\csepp{h\rbr{\xb} = \thetab^{\top} \xb}{\Pb_{iv} \thetab = \thetab}   
\end{align*}
Then with $\thetab^* \in \row(\Xb)$ from \Cref{ass:case_ood_sample}, 
\begin{align*}
    \wh{\thetab}^{dac} - \thetab^*
    = 
    \frac{1}{N} \wh{\Sigmab}_{\Xb_{iv}}^{\dagger} \Pb_{iv} \Xb^{\top} (\Xb \Pb_{iv} \thetab^* + \zb) - \thetab^*
    =
    \frac{1}{N} \wh{\Sigmab}_{\Xb_{iv}}^{\dagger} \Pb_{iv} \Xb^{\top} \zb,
\end{align*}
where $\wh{\Sigmab}_{\Xb_{iv}} \triangleq \frac{1}{N} \Pb_{iv} \Xb^{\top} \Xb \Pb_{iv}$.
Since $\wh{\thetab}^{dac} - \thetab^* \in \col\rbr{\Sb_{iv}}$, we have
$\E_{P^t}\sbr{z \cdot \xb^{\top} \Pb_e (\wh{\thetab}^{dac} - \thetab^*)} = 0$.
Therefore, let $\bs{\Sigma}_{\xb,t} \triangleq \E_{P^t}[\xb\xb^{\top}]$, with high probability,
\begin{align*}
    E_{P^s} \sbr{L_t(\wh{\thetab}^{dac}) - L_t(\thetab^*)} 
    = \ 
    & E_{P^s} \sbr{\frac{1}{2} \norm{\wh{\thetab}^{dac} - \thetab^*}_{\Sigmab_{\xb,t}}^2} 
    \\ = \ 
    & \tr \rbr{
    \frac{1}{2N} \E_{P^s}\sbr{\zb \zb^{\top}}\ 
    \E_{P^s}\sbr{\rbr{\frac{1}{N} \Pb_{iv} \Xb^{\top} \Xb \Pb_{iv}}^{\dagger}}
    \ \Sigmab_{\xb,t}}
    \\ = \ 
    & \tr \rbr{
    \frac{\sigma^2}{2N}\ 
    \E_{P^s} \sbr{\wh{\Sigmab}_{\Xb_{iv}}^{\dagger}}\ 
    \Sigmab_{\xb,t} }
    \\ \leq \ 
    & C_{t,iv}\ 
    \frac{\sigma^2}{2N}\ 
    tr \rbr{\E_{P^s} \sbr{\wh{\Sigmab}_{\Xb_{iv}}^{\dagger}}} \quad \rbr{\t{\Cref{lemma:sample-population-covariance}},\ \emph{w.h.p.}}
    \\ \lesssim \ 
    & \frac{\sigma^2}{2N} \tr\rbr{\rbr{\E_{P^s} \sbr{\Pb_{iv} \xb \xb^{\top} \Pb_{iv}}}^{\dagger}}
    \\ \leq \ 
    & \frac{\sigma^2 d_{iv}}{2N c} 
    \ \lesssim \  
    \frac{\sigma^2 d_{iv}}{2N}.
\end{align*}
\end{proof}

\begin{proof}[Proof of \Cref{thm:domain-adaption-plain-data-aug}]
Let $\wh{\Sigmab}_{\wt\Acal\rbr{\Xb}} \triangleq \frac{1}{(1+\alpha)N} \wt\Acal\rbr{\Xb}^{\top} \wt\Acal\rbr{\Xb}$.
Then with $\thetab^* \in \row(\Xb)$ from \Cref{ass:case_ood_sample}, we have $\thetab^* = \wh{\Sigmab}_{\wt\Acal\rbr{\Xb}}^{\dagger} \wh{\Sigmab}_{\wt\Acal\rbr{\Xb}} \thetab^*$. 
Since $\thetab^* \in \col\rbr{\Sb_{iv}}$, 
$\wt\Mb \Xb \thetab^* = \wt\Mb \Xb \Pb_{iv} \thetab^* = \wt\Acal(\Xb) \thetab^*$.
Then, the ERM on the augmented training set yields 
\begin{align*}
    \wh{\thetab}^{\herm} - \thetab^*
    \ = \ 
    & \frac{1}{(1+\alpha) N} \wh{\Sigmab}_{\wt\Acal\rbr{\Xb}}^{\dagger} \wt\Acal(\Xb)^{\top} \wt\Mb (\Xb \thetab^* + \zb) 
    -
    \wh{\Sigmab}_{\wt\Acal\rbr{\Xb}}^{\dagger} \wh{\Sigmab}_{\wt\Acal\rbr{\Xb}} \thetab^*
    \\ =  
    & \frac{1}{(1+\alpha) N} \wh{\Sigmab}_{\wt\Acal\rbr{\Xb}}^{\dagger} \wt\Acal(\Xb)^{\top} \wt\Mb \zb.
\end{align*}

Meanwhile with $\E_{P^t}\sbr{z \cdot \zetab_e} = \b{0}$ from \Cref{ass:case_ood_distribution}, we have $\E_{P^t}\sbr{z \cdot \Pb_e\xb} = \b{0}$. Therefore, by recalling that $\bs{\Sigma}_{\xb,t} \triangleq \E_{P^t}[\xb\xb^{\top}]$,
\begin{align*}
    L_t(\thetab) - L_t(\thetab^*)
    \ = \  
    \E_{P^t\rbr{\xb}}\sbr{\frac{1}{2} \rbr{\xb^{\top} (\thetab - \thetab^*)}^2 + 
    z \cdot \xb^{\top} \Pb_e (\thetab - \thetab^*)}
    \ = \ 
    \frac{1}{2} \norm{\thetab^* - \thetab}_{\Sigmab_{\xb,t}}^2,
\end{align*}
such that the expected excess risk can be expressed as
\begin{align*}
    \E_{P^s} \sbr{L_t(\wh{\thetab}^{\herm}) - L_t(\thetab^*)} 
    =  
    \frac{1}{2 (1+\alpha)^2 N^2} \tr \rbr{
    \E_{P^s}\sbr{
    \wh{\Sigmab}_{\wt\Acal\rbr{\Xb}}^{\dagger} 
    \rbr{\wt\Acal(\Xb)^{\top} \wt\Mb \zb \zb^{\top} \wt\Mb^{\top} \wt\Acal(\Xb) }
    \wh{\Sigmab}_{\wt\Acal\rbr{\Xb}}^{\dagger} }
    \Sigmab_{\xb,t} },
\end{align*}
where let $\wh{\Sigmab}_{\wt\Acal\rbr{\Xb_e}} \triangleq \Pb_e \wh{\Sigmab}_{\wt\Acal\rbr{\Xb}} \Pb_e$,
\begin{align*}
    & \E_{P^s}\sbr{ 
    \wh{\Sigmab}_{\wt\Acal\rbr{\Xb}}^{\dagger} 
    \rbr{\wt\Acal(\Xb)^{\top} \wt\Mb \zb \zb^{\top} \wt\Mb^{\top} \wt\Acal(\Xb) }
    \wh{\Sigmab}_{\wt\Acal\rbr{\Xb}}^{\dagger} }
    \\ \ageq \ 
    & \E_{P^s}\sbr{ \rbr{
    \Pb_{iv} \wh{\Sigmab}_{\wt\Acal\rbr{\Xb}}^{\dagger} \Pb_{iv}
    + \Pb_e \wh{\Sigmab}_{\wt\Acal\rbr{\Xb}}^{\dagger} \Pb_e }
    \wt\Acal(\Xb)^{\top} \wt\Mb\zb \zb^{\top} \wt\Mb^{\top} \wt\Acal(\Xb) 
    \rbr{ \Pb_{iv} \wh{\Sigmab}_{\wt\Acal\rbr{\Xb}}^{\dagger} \Pb_{iv} +
    \Pb_e \wh{\Sigmab}_{\wt\Acal\rbr{\Xb}}^{\dagger} \Pb_e
    } }
    \\ \ageq \ 
    & \sigma^2 (1+\alpha)^2 N \cdot 
    \E_{P^s}\sbr{\wh{\Sigmab}_{\Xb_{iv}}^{\dagger}}
    + 
    \E_{P^s}\sbr{ 
    \wh{\Sigmab}_{\wt\Acal\rbr{\Xb_e}}^{\dagger} 
    \wt\Acal(\Xb_e)^{\top} \wt\Mb \zb\zb^{\top} \wt\Mb^{\top} \wt\Acal(\Xb_e) 
    \wh{\Sigmab}_{\wt\Acal\rbr{\Xb_e}}^{\dagger} }.
\end{align*}
We denote 
\begin{align*}
    \EER_e \triangleq \ 
    \tr\rbr{
    \E_{P^s} \sbr{ \frac{1}{2 (1+\alpha)^2 N^2}
    \wh{\Sigmab}_{\wt\Acal\rbr{\Xb_e}}^{\dagger} 
    \wt\Acal(\Xb_e)^{\top} \wt\Mb \zb\zb^{\top} \wt\Mb^{\top} \wt\Acal(\Xb_e) 
    \wh{\Sigmab}_{\wt\Acal\rbr{\Xb_e}}^{\dagger} } },
\end{align*}
and observe that
\begin{align*}
    \EER_e = \E_{P^s} \sbr{ 
    \frac{1}{2} \norm{\frac{1}{(1+\alpha) N} 
    \wh{\Sigmab}_{\wt\Acal\rbr{\Xb_e}}^{\dagger} 
    \wt\Acal(\Xb_e)^{\top} \wt\Mb \zb }_2^2 } 
    > 0.
\end{align*}
Finally, we complete the proof by partitioning the lower bound for the target expected excess risk of $\wh{\thetab}^{\herm}$ into the invariantand environmental parts such that
\begin{align*}
    & \E_{P^s}\sbr{L_t(\wh{\thetab}^{\herm}) - L_t(\thetab^*)} 
    \\ \geq \ 
    & \underbrace{ \tr \rbr{\frac{\sigma^2}{2N}\ 
    \E_{P^s} \sbr{\wh{\Sigmab}_{\Xb_{iv}}^{\dagger}} 
    \Sigmab_{\xb,t} } 
    }_{=\E\sbr{L_t(\wh{\thetab}^{dac}) - L_t(\thetab^*)} }
    \ \\
    & + \
    \underbrace{\tr \rbr{ 
    \E_{P^s} \sbr{ \frac{1}{2 (1+\alpha)^2 N^2}
    \wh{\Sigmab}_{\wt\Acal\rbr{\Xb_e}}^{\dagger} 
    \wt\Acal(\Xb_e)^{\top} \wt\Mb \zb\zb^{\top} \wt\Mb^{\top} \wt\Acal(\Xb_e) 
    \wh{\Sigmab}_{\wt\Acal\rbr{\Xb_e}}^{\dagger} } 
    \Sigmab_{\xb,t} } 
    }_{\t{expected excess risk from environmental feature subspace} \geq c_{t,e} \cdot \EER_e}
    \\ \geq \ 
    & \E_{P^s} \sbr{L_t(\wh{\thetab}^{dac}) - L_t(\thetab^*)}
    + c_{t,e} \cdot \EER_e.
\end{align*}
\end{proof}

Now we construct a specific domain adaptation example with a large separation ($\ie$, proportional to $d_e$) in the target excess risk between learning with the DAC regularization ($\ie$, $\wh{\thetab}^{dac}$) and with the ERM on augmented training set ($\ie$, $\wh{\thetab}^{\herm}$). 
\begin{example}
\label{example:ood}
We consider $P^s$ and $P^t$ that follow the same set of relations in \Cref{fig:data-distribution-causal-graph}, except for the distributions over $\eb$ where $P^s\rbr{\eb} \neq P^t\rbr{\eb}$. 
Precisely, let the environmental feature $\zetab_e$ depend on $(\zetab_{iv}, y, \eb)$:
\begin{align*}
    \zetab_e = \sgn\rbr{y - \rbr{\thetab^*}^{\top} \Sb_{iv} \zetab_{iv}} \eb = \sgn(z) \eb,
    \quad
    z \sim \Ncal(0, \sigma^2),
    \quad 
    z \perp \eb,
\end{align*}    
where $\eb \sim \Ncal\rbr{\b0, \Ib_{d_e}}$ for $P^s(\eb)$ and $\eb \sim \Ncal\rbr{\b0, \sigma_t^2 \Ib_{d_e}}$ for $P^t(\eb)$, $\sigma_t \geq c_{t,e}$ (recall $c_{t,e}$ from \Cref{ass:case_ood_distribution}).
Assume that the training set $\Xb$ is sufficiently large, $n \gg d_e + \log\rbr{1/\delta}$ for some given $\delta \in (0,1)$.
Augmenting $\Xb$ with a simple by common type of data augmentations -- the linear transforms, we let
\begin{align*}
    \wt\Acal(\Xb) = \sbr{\Xb; \rbr{\Xb\Ab_1}; \dots; \rbr{\Xb\Ab_\alpha}},
    \quad
    \Ab_j = \Pb_{iv} + \ub_j \vb_j^{\top},
    \quad
    \ub_j,\vb_j \in \col\rbr{\Sb_e}
    \quad \forall\ j \in [\alpha],
\end{align*}
and define
\begin{align*}
    \nu_1 \triangleq \max \cbr{1} \cup \csepp{\sigma_{\max}(\Ab_j)}{j\in[\alpha]}
    \quad \t{and} \quad 
    \nu_2 \triangleq \sigma_{\min}\rbr{\frac{1}{1+\alpha} \rbr{\Ib_{d} + \sum_{j=1}^{\alpha} \Ab_k}},
\end{align*}
where $\sigma_{\min}(\cdot)$ and $\sigma_{\max}(\cdot)$ refer to the minimum and maximum singular values, respectively.
Then under \Cref{ass:case_ood_distribution} and \Cref{ass:case_ood_sample}, with constant probability,
\[
\E_{P^s}\sbr{L_t(\wh{\thetab}^{\herm}) - L_t(\thetab^*)} 
\gtrsim
\E_{P^s}\sbr{L_t(\wh{\thetab}^{dac}) - L_t(\thetab^*)} + 
c_{t,e} \cdot \frac{\sigma^2 d_e}{2N}. 
\]
\end{example}

\begin{proof}[Proof of \Cref{example:ood}]
With the specified distribution, for $\Eb = \sbr{\eb_1; \dots; \eb_N} \in \R^{N \times d_e}$,
\begin{align*}
    &\wh{\Sigmab}_{\wt\Acal\rbr{\Xb_e}} = 
    \frac{1}{(1+\alpha) N} 
    \Sb_e \rbr{\Eb^{\top} \Eb + \sum_{j=1}^{\alpha} \Ab_j^{\top} \Eb^{\top} \Eb
    \Ab_j} \Sb_e^{\top}
    \aleq
    \frac{\nu_1^2}{N} \Sb_e \Eb^{\top} \Eb \Sb_e^{\top},
    \\
    &\frac{1}{(1+\alpha) N} \wt\Acal(\Xb_e)^{\top} \wt\Mb\zb = 
    \rbr{\frac{1}{1+\alpha} \rbr{\Ib_{d} + \sum_{j=1}^{\alpha} \Ab_j}}^{\top}
    \frac{1}{N} \Sb_e \Eb^{\top} \abs{\zb}.
\end{align*}
By \Cref{lemma:sample-population-covariance}, under \Cref{ass:case_ood_distribution} and \Cref{ass:case_ood_sample}, we have that with high probability, $0.9 \Ib_{d_e} \aleq \frac{1}{N} \Eb^{\top} \Eb \aleq 1.1 \Ib_{d_e}$.
Therefore with $\Eb$ and $\zb$ being independent, 
\begin{align*}
    & \EER_e 
    \ = \ 
    \E_{P^s} \sbr{ 
    \frac{1}{2} \norm{\frac{1}{(1+\alpha) N} 
    \wh{\Sigmab}_{\wt\Acal\rbr{\Xb_e}}^{\dagger} 
    \wt\Acal(\Xb_e)^{\top} \wt\Mb \zb }_2^2 } 
    \\ \geq \ 
    & \frac{\sigma^2}{2N}\ \frac{\nu_2^2}{\nu_1^4}\ 
    \tr\rbr{\E_{P^s} \sbr{ 
    \rbr{\frac{1}{N} \Sb_e \Eb^{\top} \Eb \Sb_e^{\top}}^{\dagger} } }
    \\ \gtrsim 
    &\frac{\sigma^2}{2N}\ \frac{\nu_2^2}{\nu_1^4} d_e
    \\ \gtrsim \ 
    & \frac{\sigma^2 d_e}{2N},
\end{align*}
and the rest follows from \Cref{thm:domain-adaption-plain-data-aug}.
\end{proof}

\section{Technical Lemmas}

\begin{lemma}
\label{lemma:sample-population-covariance}
We consider a random vector $\xb \in \R^d$ with $\E[\xb]=\b{0}$, $\E[\xb\xb^{\top}] = \Sigmab$, and $\overline{\xb} = \Sigmab^{-1/2} \xb$ 
\footnote{In the case where $\Sigmab$ is rank-deficient, we slightly abuse the notation such that $\Sigmab^{-1/2}$ and $\Sigmab^{-1}$ refer to the respective pseudo-inverses.}
being $\rho^2$-subgaussian.
Given an $\iid$ sample of $\xb$, $\Xb=[\xb_1,\dots,\xb_n]^{\top}$, for any $\delta \in (0,1)$, if $n \gg \rho^4 d$, then $0.9 \Sigmab \aleq \frac{1}{n}\Xb^{\top} \Xb \aleq 1.1 \Sigmab$ with high probability.
\end{lemma}

\begin{proof}
We first denote $\Pb_{\Xcal} \triangleq \Sigmab \Sigmab^{\pinv}$ as the orthogonal projector onto the subspace $\Xcal \subseteq \R^d$ supported by the distribution of $\xb$.
With the assumptions $\E[\xb]=\b{0}$ and $\E[\xb\xb^{\top}] = \Sigmab$, we observe that $\E \sbr{\overline{\xb}} = \b{0}$ and $\E\sbr{\overline{\xb} \overline{\xb}^{\top}} = \E \sbr{\xb \Sigmab^{-1} \xb^{\top}} = \Pb_{\Xcal}$.
Given the sample set $\Xb$ of size $n \gg \rho^4\rbr{d+\log(1/\delta)}$ for some $\delta \in (0,1)$, we let $\Ub = \frac{1}{n} \sum_{i=1}^n \xb_i \Sigmab^{-1} \xb_i^{\top} - \Pb_{\Xcal}$. 
Then the problem can be reduced to showing that, with probability at least $1-\delta$, $\norm{\Ub}_2 \leq 0.1$.
For this, we leverage the $\eps$-net argument as following.

For an arbitrary $\vb \in \Xcal \cap\ \mathbb{S}^{d-1}$, we have
\begin{align*}
    \vb^{\top} \Ub \vb = 
    \frac{1}{n} \sum_{i=1}^n 
    \rbr{\vb^{\top} \xb_i \Sigmab^{-1} \xb_i^{\top} \vb - 1} = 
    \frac{1}{n} \sum_{i=1}^n 
    \rbr{\rbr{\vb^{\top} \overline{\xb}_i}^2 - 1},
\end{align*}
where, given $\overline{\xb}_i$ being $\rho^2$-subgaussian, 
$\vb^{\top} \overline{\xb}_i$ is $\rho^2$-subgaussian. 
Since 
\begin{align*}
    \E \sbr{\rbr{\vb^{\top} \overline{\xb}_i}^2} = \vb^{\top} \E \sbr{\overline{\xb}_i \overline{\xb}_i^{\top}} \vb = 1,
\end{align*}
we know that $\rbr{\vb^{\top} \overline{\xb}_i}^2-1$ is $16\rho^2$-subexponential.
Then, we recall the Bernstein's inequality,
\begin{align*}
    \PP\sbr{\abs{\vb^{\top} \Ub \vb} > \eps} \leq 
    2 \exp \rbr{-\frac{n}{2} 
    \min\rbr{\frac{\eps^2}{\rbr{16 \rho^2}^2}, 
    \frac{\eps}{16 \rho^2}}}.
\end{align*}

Let $N \subset \Xcal \cap\ \mathbb{S}^{d-1}$ be an $\eps_1$-net such that $\abs{N} = e^{O\rbr{d}}$. 
Then for some $0 < \eps_2 \leq 16 \rho^2$, by the union bound,
\begin{align*}
    \PP \sbr{\underset{\vb \in N}{\max}: \abs{\vb^{\top} \Ub \vb} > \eps_2} 
    \leq \ 
    & 2 \abs{N} \exp \rbr{-\frac{n}{2}
    \min\rbr{\frac{\eps_2^2}{\rbr{16 \rho^2}^2}, 
    \frac{\eps_2}{16 \rho^2}} } 
    \\ \leq \ 
    & \exp \rbr{O\rbr{d} -\frac{n}{2} \cdot \frac{\eps_2^2}{\rbr{16 \rho^2}^2} } \leq \delta
\end{align*}
whenever $n > \frac{2 \rbr{16 \rho^2}^2}{\eps_2^2} \rbr{\Theta\rbr{d} + \log\frac{1}{\delta}}$. By taking $\delta = \exp\rbr{-\frac{1}{4}\rbr{\frac{\eps_2}{16 \rho^2}}^2 n}$, we have that $\underset{\vb \in N}{\max}\abs{\vb^{\top} \Ub \vb} \leq \eps_2$ with high probability when $n > 4 \rbr{\frac{16 \rho^2}{\eps_2}}^2 \Theta\rbr{d}$, and taking $n \gg \rho^4 d$ is sufficient.

Now for any $\vb \in \Xcal \cap\ \mathbb{S}^{d-1}$, there exists some $\vb' \in N$ such that $\norm{\vb - \vb'}_2 \leq \eps_1$. 
Therefore,
\begin{align*}
    \abs{\vb^{\top} \Ub \vb} 
    \ = \
    & \abs{\vb'^{\top} \Ub \vb' + 
    2 \vb'^{\top} \Ub \rbr{\vb - \vb'} + 
    \rbr{\vb - \vb'}^{\top} \Ub \rbr{\vb - \vb'} }
    \\ \leq \ 
    & \rbr{ \underset{\vb \in N}{\max}: \abs{\vb^{\top} \Ub \vb} } + 
    2 \norm{\Ub}_2 \norm{\vb'}_2 \norm{\vb-\vb'}_2 + 
    \norm{\Ub}_2 \norm{\vb-\vb'}_2^2
    \\ \leq \ 
    & \rbr{ \underset{\vb \in N}{\max}: \abs{\vb^{\top} \Ub \vb} } + 
    \norm{\Ub}_2
    \rbr{2 \eps_1 + \eps_1^2}.
\end{align*}
Taking the supremum over $\vb \in \mathbb{S}^{d-1}$, with probability at least $1-\delta$,
\begin{align*}
    \underset{\vb \in \Xcal \cap\ \mathbb{S}^{d-1}}{\max}: \abs{\vb^{\top} \Ub \vb}
    = 
    \norm{\Ub}_2
    \leq 
    \eps_2 + \norm{\Ub}_2 \rbr{2 \eps_1 + \eps_1^2},
    \qquad
    \norm{\Ub}_2
    \leq 
    \frac{\eps_2}{2 - \rbr{1+\eps_1}^2}.
\end{align*}
With $\eps_1 = \frac{1}{3}$ and $\eps_2 = \frac{1}{45}$, we have $\frac{\eps_2}{2 - \rbr{1+\eps_1}^2} = \frac{1}{10}$.

Overall, if $n \gg \rho^4 d$, then with high probability, we have $\norm{\Ub}_2 \leq 0.1$.
\end{proof}

\begin{lemma}\label{lemma:tech_gaussian_width_lipschitz}
Let $U \subseteq \R^d$ be an arbitrary subspace in $\R^d$, and $\gb \sim \Ncal\rbr{\b0,\Ib_d}$ be a Gaussian random vector. Then for any continuous and $C_l$-Lipschitz function $\varphi: \R \to \R$ ($\ie$, $\abs{\varphi(u)-\varphi(u')} \leq C_l \cdot \abs{u-u'}$ for all $u,u' \in \R$), 
\begin{align*}
    \E_{\gb} \sbr{\underset{\ub \in U}{\sup}\ \gb^{\top} \varphi(\ub)} \leq C_l \cdot \E_{\gb} \sbr{\underset{\ub \in U}{\sup}\ \gb^{\top} \ub},
\end{align*}
where $\varphi$ acts on $\ub$ entry-wisely, $\rbr{\varphi(\ub)}_j = \varphi(u_j)$.
In other words, the Gaussian width of the image set $\varphi(U) \triangleq \cbr{\varphi(\ub) \in \R^d ~|~ \ub \in U}$ is upper bounded by that of $U$ scaled by the Lipschitz constant.
\end{lemma} 

\begin{proof}
\begin{align*}
    \E_{\gb} \sbr{\underset{\ub \in U}{\sup}\ \gb^{\top} \varphi(\ub)}
    =
    & \frac{1}{2} \E_{\gb} \sbr{\underset{\ub \in U}{\sup}\ \gb^{\top} \varphi(\ub) + \underset{\ub' \in U}{\sup}\ \gb^{\top} \varphi(\ub)} \\
    =
    & \frac{1}{2} \E_{\gb} \sbr{\underset{\ub,\ub' \in U}{\sup}\ \gb^{\top} \rbr{\varphi(\ub)-\varphi(\ub')}} \\
    \leq 
    & \frac{1}{2} \E_{\gb} \sbr{\underset{\ub,\ub' \in U}{\sup}\ \sum_{j=1}^d \abs{g_j} \abs{\varphi(u_j)-\varphi(u'_j)}} 
    \quad 
    \rbr{\t{since}\ \varphi\ \t{is $C_l$-Lipschitz}} \\
    \leq 
    & \frac{C_l}{2} \E_{\gb} \sbr{\underset{\ub,\ub' \in U}{\sup}\ \sum_{j=1}^d \abs{g_j} \abs{u_j-u'_j}} \\
    = 
    & \frac{C_l}{2} \E_{\gb} \sbr{\underset{\ub,\ub' \in U}{\sup}\ \gb^{\top} \rbr{\ub-\ub'}} \\
    = 
    & \frac{C_l}{2} \E_{\gb} \sbr{\underset{\ub \in U}{\sup}\ \gb^{\top} \ub + \underset{\ub' \in U}{\sup}\ \gb^{\top} \rbr{-\ub'}} \\
    = 
    & C_l \cdot \E_{\gb} \sbr{\underset{\ub \in U}{\sup}\ \gb^{\top} \ub}
\end{align*}
\end{proof} 

\section{Experiment Details}\label{apdx:exp_detail}

In this section, we provide the details of our experiments. Our code is adapted from the publicly released repo: \url{https://github.com/kekmodel/FixMatch-pytorch}.

\textbf{Dataset:} Our training dataset is derived from CIFAR-100, where the original dataset contains 50,000 training samples of 100 different classes. Out of the original 50,000 samples, we randomly select 10,000 labeled data as training set (i.e., 100 labeled samples per class). To see the impact of different training samples, we also trained our model with dataset that contains 1,000 and 20,000 samples. Evaluations are done on standard test set of CIFAR-100, which contains 10,000 testing samples. 

\textbf{Data Augmentation:} During the training time, given a training batch, we generate corresponding augmented samples by RandAugment \citep{cubuk2020randaugment}. We set the number of augmentations per sample to 7, unless otherwise mentioned.

To generate an augmented image, the RandAugment draws $n$ transformations uniformaly at random from 14 different augmentations, namely \{identity, autoContrast, equalize, rotate, solarize, color, posterize, contrast, brightness, sharpness, shear-x, shear-y, translate-x, translate-y\}. The RandAugment provides each transformation with a single scalar (1 to 10) to control the strength of each of them, which we always set to 10 for all transformations. By default, we set $n=2$ (i.e., using 2 random transformations to generate an augmented sample). To see the impact of different augmentation strength, we choose $n\in\cbr{1,2, 5, 10}$. Examples of augmented samples are shown in \Cref{fig:augmentation_strength}.

\textbf{Parameter Setting:} The batch size is set to 64 and the entire training process takes $2^{15}$ steps. During the training, we adopt the SGD optimizer with momentum set to 0.9, with learning rate for step $i$ being $0.03 \times \cos{\rbr{\frac{i \times 7\pi}{2^{15}\times 16}}}$.  

\textbf{Additional Settings for the semi-supervised learning results:} For the results on semi-supervised learning, besides the 10,000 labeled samples, we also draw additionally samples (ranging from 5,000 to 20,000) from the training set of the original CIFAR-100. We remove the labels of those additionally sampled images, as they serve as ``unlabeled" samples in the semi-supervised learning setting. The FixMatch implementation follows the publicly available on in \url{https://github.com/kekmodel/FixMatch-pytorch}.

\end{document}